\documentclass{article}

\PassOptionsToPackage{numbers, sort&compress}{natbib}
\usepackage[preprint]{Aux/neurips_2026}

\usepackage[utf8]{inputenc} 
\usepackage[T1]{fontenc}    
\usepackage{hyperref}       
\usepackage{url}            
\usepackage{xurl}           
\usepackage{booktabs}       

\usepackage{nicefrac}       
\usepackage{microtype}      
\usepackage[dvipsnames, table]{xcolor} 
\usepackage{amsfonts,amssymb,amsthm,amsmath,amscd}

\usepackage{amsmath}
\usepackage{amssymb}
\usepackage{algorithm}
\usepackage[noend]{algpseudocode}

\newcommand{\R}{\mathbb{R}}

\newcommand{\llangle}{\langle \! \langle}
 \newcommand{\rrangle}{\rangle \! \rangle}
\usepackage{fontawesome5}
\newcommand{\bruno}[1]{}
\newcommand{\jake}[1]{}
\newcommand{\kelly}[1]{}
\newcommand{\philipp}[1]{}
\newcommand{\katya}[1]{}
\newcommand{\vincenzo}[1]{}

\usepackage{colortbl}  

\providecommand{\bestcell}[1]{\begingroup\setlength{\fboxsep}{1.2pt}\colorbox{yellow!30}{\bfseries #1}\endgroup}

\usepackage{array}
\usepackage{booktabs}

\newcolumntype{G}{>{\columncolor{gray!15}\centering\arraybackslash}p{0.09\linewidth}}

\newcommand{\npfdelta}[1]{\textcolor{green!50!black}{+#1}}
\newcommand{\modeldelta}[1]{\textcolor{red!70!black}{-#1}}
\newcommand{\zerodelta}{\textcolor{green!50!black}{---}}
\usepackage{enumerate}

\usepackage{graphicx}
\usepackage{overpic}
\usepackage{float}
\usepackage{wrapfig}
\usepackage[most]{tcolorbox}
\usepackage{subcaption}
\usepackage{thm-restate}
\tcbuselibrary{skins, breakable}

\usepackage{adjustbox}
\usepackage{booktabs}

\usepackage[most]{tcolorbox}
\usepackage{amsthm}

\newtcbtheorem[number within=section]{informalthm}{Informal Theorem}%
{
  enhanced,
  breakable,
  colback=gray!2,
  colframe=gray!25,
  coltitle=black,
  fonttitle=\bfseries,
  boxrule=0.3pt,
  arc=2pt,
  left=6pt,
  right=6pt,
  top=6pt,
  bottom=6pt,
  before skip=10pt,
  after skip=10pt,
}{ithm}

\definecolor{intuitionblue}{RGB}{235,245,255}
\definecolor{intuitionborder}{RGB}{70,120,180}

  \newcommand{\kfc}{\ensuremath{k}\textnormal{FC}}
  \newcommand{\kforms}{\ensuremath{k}\textnormal{-forms}}
  \newcommand{\cdc}{\textit{carré du champ}}
  \newcommand{\CdC}{\textit{Carré du Champ}}

\newtcolorbox{intuitionbox}{
  colback=intuitionblue,
  colframe=intuitionborder,
  boxrule=0.6pt,
  arc=2pt,
  left=6pt,
  right=6pt,
  top=5pt,
  bottom=5pt,
  fonttitle=\bfseries,
  title=Intuition: Volume-based learning.
}

\usepackage{array}
\usepackage{booktabs}
\usepackage{longtable}
\usepackage{makecell}

\usepackage{tikz}
\usetikzlibrary{backgrounds, calc, fit, positioning, shapes.geometric, decorations.pathmorphing, decorations.pathreplacing, arrows.meta}
\usepackage{tikz-cd}
\usepackage{tikzit}
\usetikzlibrary{shapes.geometric}

\tikzstyle{red dot}=[fill={rgb,255: red,191; green,0; blue,64}, draw=black, shape=circle]
\tikzstyle{green dot}=[fill={rgb,255: red,0; green,128; blue,128}, draw=black, shape=circle]
\tikzstyle{rectangle}=[fill=white, draw=black, shape=rectangle]
\tikzstyle{triangle}=[fill=white, draw=black, regular polygon sides=3, regular polygon, shape=regular polygon, tikzit shape=regular polygon, shape border rotate=180, minimum size=5mm, inner sep=0pt]
\tikzstyle{bullet}=[fill=black, draw=black, shape=circle, minimum size=2mm, inner sep=0pt]
\tikzstyle{noborder}=[fill=none, draw=none, shape=circle]
\tikzstyle{widerectangle}=[fill=white, draw=black, shape=rectangle, minimum width=16mm, minimum height=4mm, inner sep=1pt, tikzit shape=rectangle]

\tikzstyle{new edge style 0}=[-]

\definecolor{arroworange}{RGB}{255,161,90}
\definecolor{arrowpink}{RGB}{225,176,233}

\usepackage{extpfeil}
\usepackage{url}
\usepackage{comment}
\usepackage{parskip}
\usepackage{pifont}
\usepackage{mathrsfs}
\usepackage{scalerel}

\let\todo\undefined
\usepackage[colorinlistoftodos]{todonotes}

\let\todo\undefined

\newcommand{\todo}[1][]{}
\newcommand{\tocite}[1][]{}
\newcommand{\toref}[1][]{}

\newtheorem{theorem}{Theorem}
\newtheorem{lemma}[theorem]{Lemma}
\newtheorem{proposition}[theorem]{Proposition}

\newtheorem{corollary}[theorem]{Corollary}
\theoremstyle{remark}
\newtheorem{remark}{Remark}

\theoremstyle{plain}
\colorlet{necolor}{gray}

\usepackage{mdframed}

\mdfdefinestyle{nonessentialstyle}{
  topline           = false,
  bottomline        = false,
  rightline         = false,
  leftline          = true,
  linewidth         = 0.5pt,
  linecolor         = necolor!60,
  backgroundcolor   = white,
  innerleftmargin   = 6pt,
  innerrightmargin  = 2pt,
  innertopmargin    = 3pt,
  innerbottommargin = 4pt,
  splittopskip      = 6pt,
  splitbottomskip   = 4pt,
}

\newmdenv[style=nonessentialstyle]{nonessentialinner}

\title{Neural Point-Forms }

\author{%
  \begin{minipage}[t]{0.45\textwidth}
    \centering
    \textbf{Bruno Trentini}\thanks{Equal contribution.}\normalfont \\
    NVIDIA, University of Oxford \\
    London, UK
  \end{minipage}
  \begin{minipage}[t]{0.45\textwidth}
    \centering
    \textbf{Jacob Hume}\footnotemark[1] \normalfont \\
    University of Oxford \\
    Oxford, UK
  \end{minipage}
  \\[4em]
  \begin{minipage}[t]{0.45\textwidth}
    \centering
    \textbf{Vincenzo Antonio Isoldi} \\
    Max Planck Institute \\
    for Mathematics in the Sciences \\
    Leipzig, Germany
  \end{minipage}
  \begin{minipage}[t]{0.45\textwidth}
    \centering
    \textbf{Philipp Misof} \\
    Department of Mathematical Sciences \\
    Chalmers University of Technology and \\
    University of Gothenburg \\
    Gothenburg, Sweden
  \end{minipage}
  \\[6em]
  \begin{minipage}[t]{0.45\textwidth}
    \centering
    \textbf{Ekaterina S.~Ivshina} \\
    School of Engineering and Applied Sciences \\
    Harvard University \\
    Cambridge, MA, US
  \end{minipage}
  \begin{minipage}[t]{0.45\textwidth}
    \centering
    \textbf{Kelly Maggs}\thanks{Corresponding author: \texttt{maggs@mpi-cbg.de}.} \\
    Max Planck Institute \\
    of Molecular Cell Biology and Genetics \\
    Dresden, Germany
  \end{minipage}
}

\begin{document}

\maketitle

\begin{abstract}
Point cloud learning often rests on the premise that observed samples are noisy traces of an underlying geometric object, such as a manifold embedded in a high-dimensional feature space. Yet much of this geometry is not captured directly by coordinates, pairwise distances, or learned graph neighborhoods alone. In the smooth setting, differential forms are devices to encode higher order tangency information. In this work, we introduce a new family of principled learnable geometric features for point clouds called \textit{neural point-forms} (NPFs). In the absence of a natural tangency structure, we instead use Laplacian-based techniques from Diffusion Geometry to build a discrete model for comparing differential forms on point clouds via inner products. In the continuum, submanifolds of a shared ambient feature space are represented as comparison matrices, whose entries describe how pairs of feature forms interact with extrinsic tangency information. We make this intuition precise by proving the long-run consistency of comparison matrices under standard sampling, bandwidth, density, and manifold-hypothesis assumptions. This yields a compact, efficient and permutation-invariant neural layer whose output is a learned form-comparison matrix. Across synthetic and biologically relevant experiments, we show that NPFs provide a competitive, and interpretable representation, with the strongest benefits appearing when labels depend on sampling density, manifold-like structure, or response-relevant population geometry.

\end{abstract}


\section{Introduction}

Machine learning typically represents objects by real-valued features, treating each datum as a point in an abstract Euclidean feature space $\mathbb{R}^D$. In many applications, however, a datum is more naturally described not by a single vector but by a point cloud -- a population of vectors occupying a region of feature space. A tissue, for example, is characterized by the joint distribution of its cells across relevant biological features, not by any individual cell. Such settings are inherently \textit{extrinsic}: how a population sits within the ambient coordinates is as informative as intrinsic quantities like pairwise distances.



The manifold hypothesis makes this view precise. We model each observed point cloud as a sample from an underlying submanifold $\iota:M\hookrightarrow \mathbb{R}^D$. Different classes of point clouds correspond to different ways of occupying the same ambient space, resulting in distinct underlying submanifolds. In the tissue example, two biological conditions may share the same molecular features yet confine cell populations to distinct regions or geometries within $\mathbb{R}^D$. More generally, such submanifolds may arise from distinct latent mechanisms -- for example, biological regulation, physical constraints, or dynamical laws. The learning problem we study is therefore to classify submanifolds from their sampled point clouds, without access to the constraints that produced them.

 Such problems appear perfectly set up for Geometric Deep Learning (GDL) \cite{bronstein2021geometric}, which, as the name suggests, designs deep learning architectures using structures from geometry. The preeminent architectures in GDL are \textit{message passing graph neural networks} (GNNs)~\cite{gilmer2017neural, battaglia2018relational, hamilton2018inductiverepresentationlearninglarge, thomas2019kpconv, xu2018how, velickovic2018graph, dwivedi2021generalizationtransformernetworksgraphs}. Having matured as a field, the strengths and limitations of message passing have become apparent \cite{garg2020generalizationrepresentationallimitsgraph}. Topological Deep Learning (TDL)~\cite{papamarkou2024position}  extends this paradigm to `higher-order' relational structures such as cell complexes or hypergraphs, resolving some of these weaknesses while inheriting others \cite{ebli2020simplicial,bunch2020simplicial,bodnar2021weisfeiler,roddenberry2021principled,feng2019hypergraph,bodnar2021cellular,hajij2022topological,bodnar2022neural,hume2025sheafification, papillon2025topotuneframeworkgeneralized}.




In contrast to methods built on combinatorial structures are architectures that process and classify point clouds directly. The dominant family of point cloud classifiers takes the raw ambient coordinates of each point as input, processes each point through a shared neural network, and combines the per-point outputs into a single representation using a permutation-invariant pooling operation~\cite{qi2017pointnet,qi2017pointnet++,wang2019dynamic,thomas2019kpconv,zhao2021point,thomas2018tensor,fuchs2020se3,batatia2022mace}. Such methods, however, do not yet possess extensions to algorithms that capture higher order structure. 



\begin{wrapfigure}{r}{0.25\textwidth}
    \centering
    \includegraphics[width=0.25\textwidth]{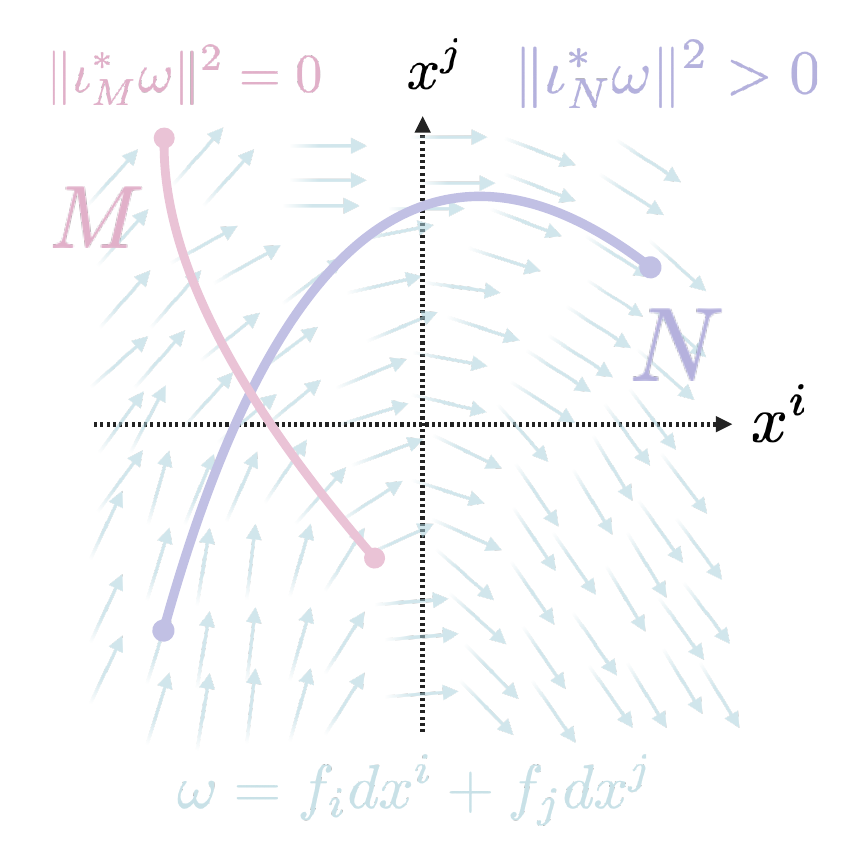}
    \caption{ \centering Pairings of \( k \)-form restrictions distinguish manifolds by tangency information.}
    \label{fig:form-norms}
\end{wrapfigure}






 Following on from the perspective in \cite{maggs2023simplicial}, the approach in this work is to use ambient differential forms as higher-order geometric features. Such features are activated or inactivated by the tangency and normal structure of a submanifold, as quantified by the inner products and norms of their restrictions (Fig. \ref{fig:form-norms}). Ambient differential forms are parametrized using neural networks, backpropagated and learned from data in practice. The content of a post-trained neural $k$-form has, in theory, learned the where $k$-sets of variables generate volume salient to the learning task. One limitation of \cite{maggs2023simplicial}, however, is that it required a simplicial structure on data, and included an expensive Riemann-sum approximation of volume integrals in the forward pass.
 



To overcome this limitation, we define \textit{Neural Point Forms} (NPFs) directly on point cloud data by leveraging an emerging paradigm in geometric data analysis of \textit{complex-free} geometric models \cite{jonesComputingDiffusionGeometry2026}. The two main examples of this paradigm, Spectral Exterior Calculus \cite{giannakis-berry2020} and Diffusion Geometry \cite{jones2024diffusiongeometry}, are based on using the Laplace-Beltrami operator $\Delta_g$ of a manifold to approximate the Riemmanian metric/inner product on differentials via the \textit{carré du champ identity}
\begin{equation} \label{smooth-carre-du-champ}
   \langle df, dh \rangle_g = \dfrac{1}{2} \Big( f \Delta_g h + h \Delta_g f - \Delta_g(fh) \Big).
\end{equation} The Laplace-Beltrami operator has a well-established approximation theory on data~\cite{coifman2006diffusion, Berry2016, Belkin2008, singer2006graph}, making the approximation of the inner products possible in practice. The motivation for this line of work is that many of the key objects in geometry can be estimated via formulae involving the Riemmanian metric, leaving the full manifold and differential structure of data `implicit' \cite{kawasaki2026diffusion}. 

The neural point-forms in this work are also defined `implicitly': we bypass defining a de Rham-like complex of differential forms on point clouds explicitly. To achieve this goal, we employ a manifold-like representation of data introduced \cite{jones2026manifolddiffusiongeometrycurvature,jonesComputingDiffusionGeometry2026} that we refer to as the \textit{Gram field} for convenience. This transformation recasts each point cloud as a parametrized collection of symmetric matrices (bilinear forms), whose evaluation approximates the inner product of ambient forms in the probablistic continuum limit, and whose kernel and image encode normal and tangency information of a submanifold. The upshot of this method is that our forward pass is extremely simple, consisting only of composing a parametrized bilinear form with the forward pass of an MLP.

To illustrate the viability of our model, we design a suite of synthetic experiments based on the extrinsic manifold learning perspective outlined above. In these, we show that we can interpret the learned forms, and reveal some somewhat surprising fail modes of traditional GNNs. Inspired by \cite{viswanath2025hiponet}, we design a new set of biologically meaningful real-world benchmarks using publically available single-cell sequenced patient-derived organoids \cite{RamosZapatero2023}, where each sample is a point cloud with biologically meaningful features. Despite being a completely novel and unstudied class of architectures, neural point-forms perform competitively with established SOTA point cloud classification, GNN, and TDL methods. 

\paragraph{Contributions} Our main contributions are as follows:
\begin{enumerate}
    \item Define \textit{Neural Point-Forms} (NPFs) as learnable geometric features of point cloud data; 
    \item Prove that the main objects are consistent with Riemannian definitions in the continuum under the manifold hypothesis;
    \item Develop an efficient algorithm for featurizing point cloud data that is competitive on real-world biological tasks.
\end{enumerate}

\section{Background on differential forms}
In this section, we give an informal and intuitive overview of differential forms and how they capture tangency information of submanifolds. We present intuitive definitions here, and provide a more rigorous exposition in Appendix \ref{app:background-diff-forms}. 

\begin{figure}[t]
    \centering
    
    \begin{subfigure}[t]{0.32\textwidth}
        \centering
        \includegraphics[width=\textwidth]{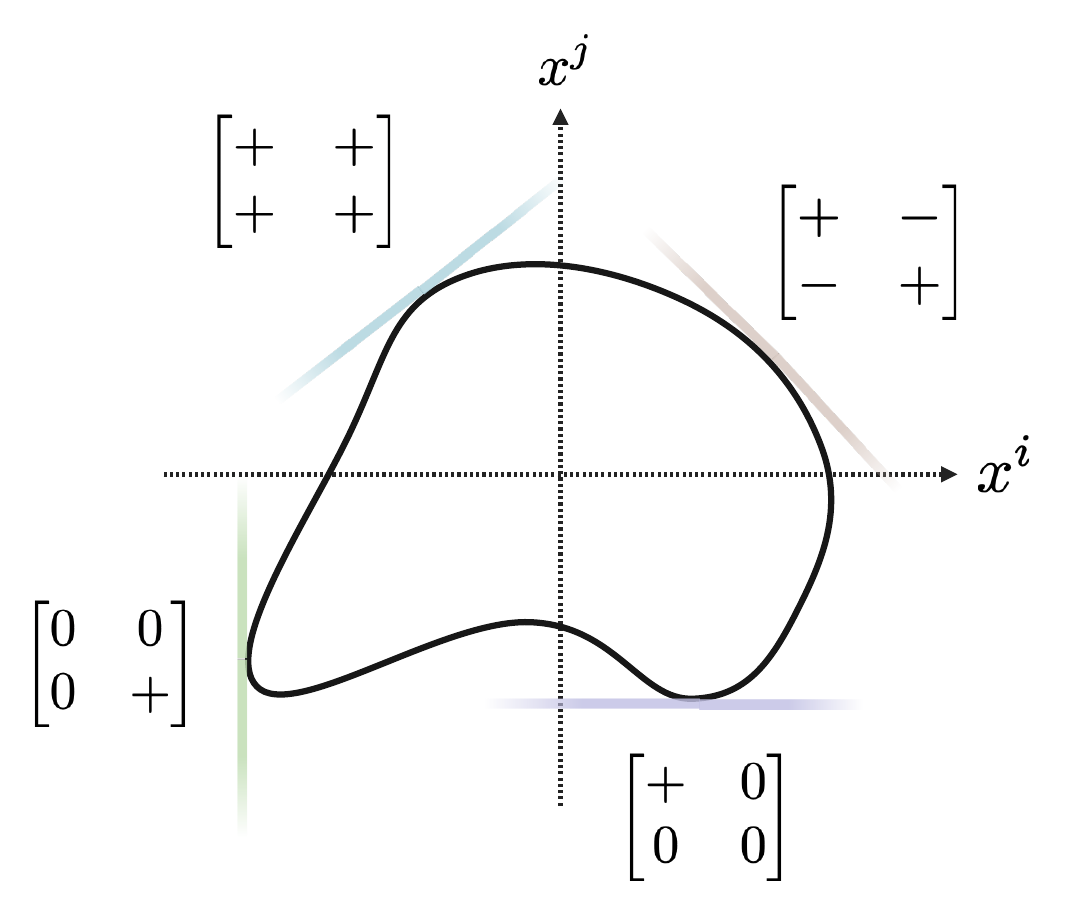}
        \caption{\centering The Gram field $G_M$ of a submanifold.}
        \label{fig:gram-field-smooth}
    \end{subfigure}
    \hspace{1em}
     \begin{subfigure}[t]{0.32\textwidth}
        \centering
        \includegraphics[width=\textwidth]{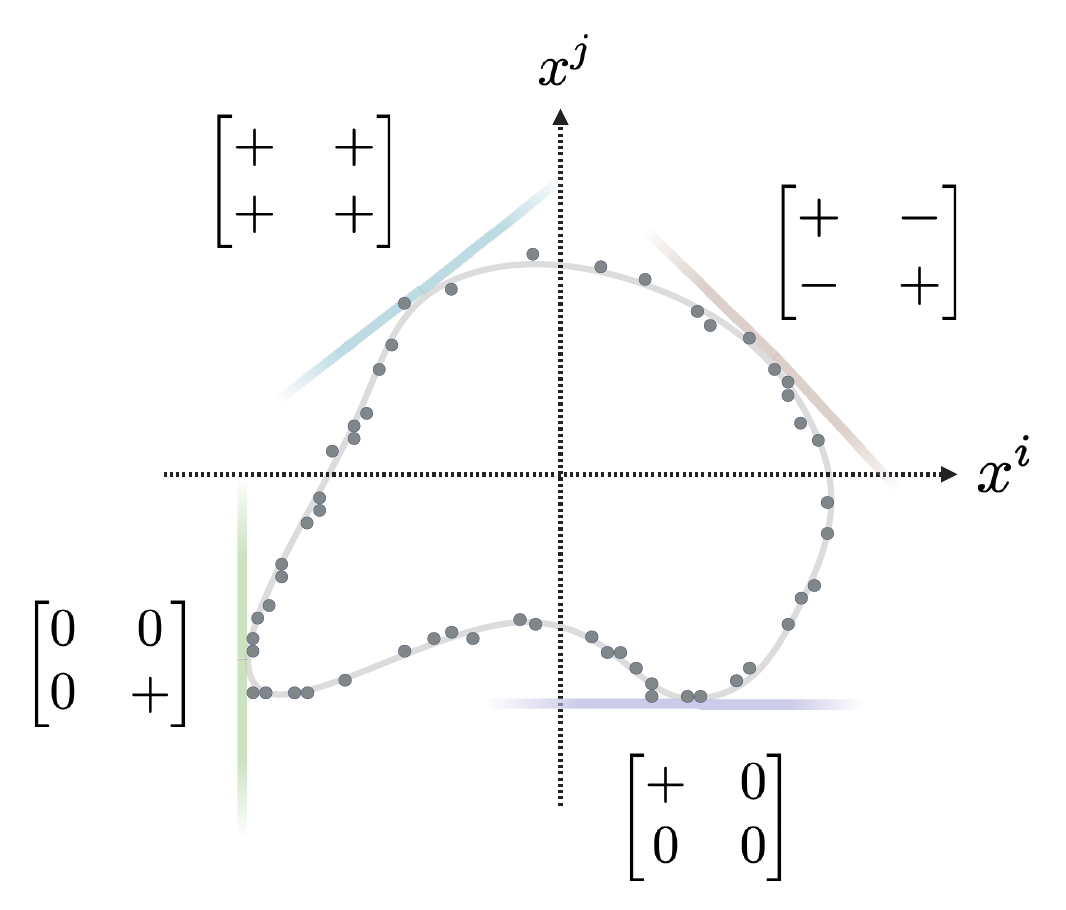}
        \caption{\centering The Gram field $G_P$ of a sample $P \subset M$.}
        \label{fig:gram-field-discrete}
    \end{subfigure}
    \caption{\centering Tangency information is encoded by images of Gram fields and their approximations. }
    \label{fig:gram-and-formnorm}
\end{figure}

\paragraph{Basis forms in $\R^D$} We first provide an informal introduction to differential $k$-forms in ambient Euclidean space $\R^D$, which have a simple and interpretable structure. The \textit{$k$-th exterior algebra $\Lambda^k(dx^1, \ldots, dx^D)$ over $\R^D$}  is the vector space generated by the basis
\begin{equation}
      \big\{\, dx^I = dx^{i_1} \wedge \ldots \wedge dx^{i_k} \mid I = (1 \leq i_1 < \ldots < i_k \leq D) \in \text{Asc}(k,D) \, \big\}.
\end{equation} Each formal symbol $dx^I$ corresponds to a family of alternating functions
\begin{equation}
    dx^I_p : T_p \R^D \times \ldots \times T_p\R^D \to \R
\end{equation} that measure the volume spanned by $k$ tangent vectors in the subspace spanned by coordinates $I = \{ i_1, i_2, \ldots, i_k\}$. In this way, each multi-index $I$ corresponds to a $k$-dimensional feature subspace whose volume is measured by $dx^I$.

\paragraph{Differential forms in $\R^D$} A \textit{differential $k$-form} $\omega \in \Omega^k(\R^D)$ is the data of
\begin{equation}
    \omega = f_I \cdot dx^I := \sum_I f_I \cdot dx^I; \qquad f_I\in C^\infty(\R^D)
\end{equation} where $C^\infty(\R^D)$ is the algebra of smooth functions on $\R^D$ and $I$ ranges over $\text{Asc}(k,D)$. Intuitively, the \textit{scaling functions} rescale oriented volume in the $I$-th coordinate subspaces at each point, i.e., they reweight the basis forms
\begin{equation}
f_I(p) \cdot dx^I_p : T_p \R^D \times \ldots \times T_p\R^D \to \R
\end{equation} by a scaling factor depending on $p \in \R^D$. By definition, the data of a differential form on $\R^D$ is completely characterized by the coordinate vector of scaling functions
\begin{equation}
    f = \Big( \, f_I \mid I \in \text{Asc}(k,D) \Big) : \R^D \to \R^{\binom{D}{k}}.
\end{equation}

\paragraph{Alignment fields} Let $\iota : M \hookrightarrow \R^D$ be an embedded manifold. Define $\Pi_p : T_p\R^D \to T_pM$ to be the orthogonal projection onto the tangent space at a point $p \in M$ and $\nabla x^i$ be the gradient vector field of the $i$-th coordinate. The \textit{(i,j)-alignment field}\footnote{Our terminology.} is the mapping
\begin{equation} \label{eq:1form-innerprod}
    p \in M \; \; \mapsto \; \; \Big\langle \iota^*dx^i, \iota^*dx^j \Big\rangle_M(p) := \Big\langle \Pi_p \nabla x^i, \Pi_p \nabla x^j \Big\rangle_{\R^D}(p)
\end{equation} and measures the correlation between feature directions once projected onto the tangent space of $M$. Here, the notation $\iota^* dx^i$ is formally the pullback along $\iota$ (See App. \ref{app:background-diff-forms}). Analogously, the \textit{$(I,J)$-alignment field} is the mapping
\begin{equation}
p \; \; \mapsto \; \; \left\langle \iota^* dx^I,\iota^* dx^J \right\rangle_{M}(p)
:=
\det \left[
\left\langle
\iota^* dx^{i_a},
\iota^* dx^{j_b}
\right\rangle_{M}(p)
\right]_{a,b=1}^k
\label{eqn:IJ-alignment-field}
\end{equation} for multi-indices $I=(i_1<\cdots<i_k)$ and $J=(j_1<\cdots<j_k)$. Such fields capture higher order information about tangential \textit{volume} --- for example, the alignment field $$ \big\langle \,\iota^*dx^I, \iota^* dx^I \,\big \rangle_M(p) = \lVert \iota^* dx^I \rVert_M^2(p) $$ measures the (squared) volume of the $k$-cube \begin{equation}
        \text{Vol}_k^2\Big( \,\Pi_p \nabla x^{i_1},\Pi_p \nabla x^{i_2}, \cdots, \Pi_p \nabla x^{i_k} \, \Big) = \det_{i_a,i_b} \Big[ \Big\langle \Pi_p \nabla x^{i_a}, \Pi_p \nabla x^{i_b} \Big\rangle_{\R^D}(p) \Big]^2
\end{equation} spanned by the $I$-feature directions projected into the tangent space at $p \in M$.

\begin{wrapfigure}{r}{0.40\textwidth}
    \centering
    \resizebox{0.40\textwidth}{!}{%
        \begin{tikzpicture}[
    cell/.style={minimum size=0.38cm, inner sep=0pt},
    gridbox/.style={draw=black, thick, inner sep=0pt},
]

\newcommand{\drawgrid}[4]{%
  \begin{scope}[shift={(#1,#2)}]
    \draw[thick] (0,0) rectangle (3*0.38, 3*0.38);

    \foreach \i in {1,2} {
      \draw ({\i*0.38},0) -- ({\i*0.38},{3*0.38});
      \draw (0,{\i*0.38}) -- ({3*0.38},{\i*0.38});
    }

    \foreach \r/\c in {#4} {
      \fill[#3] ({\c*0.38+0.02},{(2-\r)*0.38+0.02}) 
        rectangle ({\c*0.38+0.36},{(2-\r)*0.38+0.36});
    }

    \draw[thick] (-0.1, -0.06) -- (-0.1, {3*0.38+0.06});
    \draw[thick] ({3*0.38+0.1}, -0.06) -- ({3*0.38+0.1}, {3*0.38+0.06});
  \end{scope}
}

\def\bs{1.14}    
\def\gapx{0.55}  
\def\gapy{0.45}  

\pgfmathsetmacro{\colA}{0}
\pgfmathsetmacro{\colB}{\bs + \gapx}
\pgfmathsetmacro{\colC}{2*\bs + 2*\gapx}

\pgfmathsetmacro{\rowA}{2*\bs + 2*\gapy}
\pgfmathsetmacro{\rowB}{\bs + \gapy}
\pgfmathsetmacro{\rowC}{0}

\definecolor{colD1}{HTML}{9B8EC4}
\definecolor{colD2}{HTML}{D47B7B}
\definecolor{colD3}{HTML}{C4B05A}

\definecolor{colAB}{HTML}{5BA3B5}
\definecolor{colAC}{HTML}{E09B5A}
\definecolor{colBC}{HTML}{7BBF7B}

\drawgrid{\colA}{\rowA}{colD1!50}{0/0, 0/1, 1/0, 1/1}

\drawgrid{\colB}{\rowA}{colAB!50}{0/0, 0/2, 1/0, 1/2}
\drawgrid{\colA}{\rowB}{colAB!50}{0/0, 0/1, 2/0, 2/1}

\drawgrid{\colC}{\rowA}{colAC!50}{0/1, 0/2, 1/1, 1/2}
\drawgrid{\colA}{\rowC}{colAC!50}{1/0, 1/1, 2/0, 2/1}

\drawgrid{\colB}{\rowB}{colD2!50}{0/0, 0/2, 2/0, 2/2}

\drawgrid{\colC}{\rowB}{colBC!50}{0/1, 0/2, 2/1, 2/2}
\drawgrid{\colB}{\rowC}{colBC!50}{1/0, 1/2, 2/0, 2/2}

\drawgrid{\colC}{\rowC}{colD3!50}{1/1, 1/2, 2/1, 2/2}

\pgfmathsetmacro{\matH}{3*\bs + 2*\gapy}
\pgfmathsetmacro{\matW}{3*\bs + 2*\gapx}
\pgfmathsetmacro{\bracketPad}{0.15}

\draw[very thick] ({-\bracketPad}, {-\bracketPad}) 
  -- ({-\bracketPad - 0.15}, {-\bracketPad})
  -- ({-\bracketPad - 0.15}, {\matH + \bracketPad})
  -- ({-\bracketPad}, {\matH + \bracketPad});

\draw[very thick] ({\matW + \bracketPad}, {-\bracketPad}) 
  -- ({\matW + \bracketPad + 0.15}, {-\bracketPad})
  -- ({\matW + \bracketPad + 0.15}, {\matH + \bracketPad})
  -- ({\matW + \bracketPad}, {\matH + \bracketPad});

\pgfmathsetmacro{\labelY}{\matH + 0.45}
\node[font=\small] at ({\colA + \bs/2}, \labelY) {$(1,2)$};
\node[font=\small] at ({\colB + \bs/2}, \labelY) {$(1,3)$};
\node[font=\small] at ({\colC + \bs/2}, \labelY) {$(2,3)$};

\pgfmathsetmacro{\labelX}{-0.75}
\node[font=\small] at (\labelX, {\rowA + \bs/2}) {$(1,2)$};
\node[font=\small] at (\labelX, {\rowB + \bs/2}) {$(1,3)$};
\node[font=\small] at (\labelX, {\rowC + \bs/2}) {$(2,3)$};

\end{tikzpicture}
    }
    \caption{\centering The second Gram field $G^{(2)}_M$ is constructed as a compound matrix of the first $G_M$.}
    \label{fig:compound-matrices}
\end{wrapfigure}
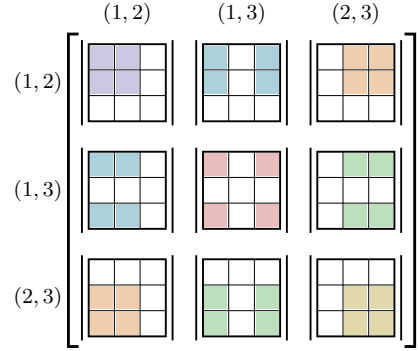

\paragraph{Gram fields} The \textit{$k$-th Gram field} is the mapping
\begin{equation}
    p \in M \, \, \mapsto \, \, G_M^{(k)}(p) := \Big[ \, \left\langle \iota^* dx^I,\iota^* dx^J \right\rangle_{M}(p) \, \Big]_{IJ} 
\end{equation} that assembles the collection of alignment fields into a parametrized symmetric matrix. As observed in \cite{jones2026manifolddiffusiongeometrycurvature, kawasaki2026diffusion}, the Gram field in dimension 1 is the matrix form of the projection operator
\begin{equation} 
    \Pi_p =G_M(p) : \R^D \to T_pM \subset \R^D
\end{equation} and thus recovers the tangent and normal spaces
\begin{equation}\label{eq:gramfield-decomposition}
    T_p M = \text{Im} \,G_M(p) \qquad  N_pM = \text{Ker} \, G_M(p)
\end{equation} as the respective image and kernel. As an object, the Gram field thus faithfully encodes the tangency and normal information of a submanifold as a family of $(D \times D)$-matrices parametrized by $p \in M$. Analogous decompositions hold for the $k$-th cotangent spaces (See App. \ref{app:background-diff-forms}, \ref{eq:k-normal-tangent-decomp}).

\paragraph{Inner products of forms} Define the local inner product on $k$-forms $\omega = f_I \cdot dx^I$ and $\eta = h_J \cdot dx^J$ by the mapping (using Einstein notation)
\begin{equation}
    p \, \, \mapsto \, \,  \Big(\, f^T G_M^{(k)} h\, \Big)(p) =  f_I(p) h_J(p)\left\langle \iota^* dx^I,\iota^* dx^J \right\rangle_{M}(p) 
\label{eqn:main-local-inner-product}
\end{equation} where $f$ and $h$ are the scaling vectors corresponding to $\omega$ and $\eta$ respectively. From this perspective, we interpret $G_M$ as a parametrized bilinear form on the scaling functions that preserves the tangential part and discards the normal part according to the decomposition in Equation \ref{eq:k-normal-tangent-decomp}. For a given measure $\mu$ on $M$, the \textit{global inner product} is given by
\begin{equation}
    \langle\!\langle \iota^* \omega, \iota^* \eta \rangle\!\rangle_M := \int_M \Big(f^T G_M^{(k)}h\Big)(p) d\mu,
\end{equation} i.e., the average local inner product with respect to $\mu$. Intuitively, the inner product scores the average alignment of submanifolds with pairs of ambient forms (Fig. 1).

\section{Neural Point-Forms}
\label{sec:formulation}


We now translate the theory above into the point cloud setting using the techniques of Diffusion Geometry \cite{jones2026manifolddiffusiongeometrycurvature, jonesComputingDiffusionGeometry2026}. We show our estimators are consistent, and detail how our constructions lead naturally to a algorithms for processing point clouds.

\paragraph{Discrete \CdC.}
Suppose now that $P \subset \mathbb{R}^D$ is a finite point cloud. Since $P$
is discrete, it does not have a nontrivial tangent or cotangent bundle. We
therefore replace the smooth Laplace--Beltrami operator by a discrete
Laplacian $L_P : \mathbb{R}^P \to \mathbb{R}^P$ where $\mathbb{R}^P$ is the set of functions $f: P \to \R$. In our implementation, $L_P$ is taken to be the discrete variable-bandwidth Laplacian, since it is well-equipped for dealing with noisy nonuniform samples in practice~\ref{discrete-variable-bandwidth-laplacian} from \cite{Berry2016}. The corresponding \textit{discrete \cdc} is the bilinear map
\begin{equation}
\Gamma_{L_P}(f,h)(p)
:=
\frac{1}{2}
\Big(
f(p)L_P h(p)
+
h(p)L_P f(p)
-
L_P(fh)(p)
\Big),
\qquad
f,h \in \mathbb{R}^P,\ p\in P,
\end{equation}
where $fh$ denotes pointwise multiplication.

\paragraph{Gram fields on point clouds} Let $x^i : P \to \mathbb{R}$ denote the restriction of the $i$th ambient
coordinate function to $P$. For $k\geq 1$, the $k$th extrinsic Gram field is the map $
G_P^{(k)} : P \to \R^{\binom{D}{k}\times \binom{D}{k}}$
whose entries are indexed by increasing multi-indices
$I=(i_1<\cdots<i_k)$ and $J=(j_1<\cdots<j_k)$ and are given by
\begin{equation}
\left(G_P^{(k)}(p)\right)^{IJ}
:=
\det
\left[
G_P^{i_a j_b}(p)
\right]_{a,b=1}^k
=
\det
\left[
\Gamma_{L_P}(x^{i_a},x^{j_b})(p)
\right]_{a,b=1}^k .
\end{equation}
Intutively, the discrete \cdc \textit{defines} the discrete alignment field, motivated by equality witnessed via the smooth \cdc formula \ref{smooth-carre-du-champ}. We thus think of $G_P^{(k)}(p)$ as the discrete analogue the smooth Gram field.

\paragraph{Discrete inner products of forms} Inner products on forms over a point cloud are defined analogously to smooth case, i.e., via the discrete Gram field as a parametrized bilinear operator. Explicitly,  $\omega = f_I \cdot dx^I$ and $\eta = h_J \cdot dx^J$ define the local inner product on $k$-forms over $P$ by the mapping
\begin{equation} \label{def:gramfield-innerproduct}
    p \, \, \mapsto \, \,  \big\langle \, \omega, \eta \, \big\rangle_P(p) := \Big(\, f^T G_P^{(k)} h\, \Big)(p).
\end{equation} Similarly, for a discrete measure $\mu_P$ on $P$, define the global inner product of $\omega$ and $\eta$ by
\begin{equation}
\label{def:gramfield-global-inner-product}
    \llangle \omega, \eta \rrangle_P = \sum_{p \in P} \langle \, \omega, \eta \, \rangle_P(p) \mu(p).
\end{equation}

\paragraph{Consistency results} We formalize the claim 

\begin{equation}
    \llangle \omega ,\eta \rrangle_P \qquad \text{is an estimator of} \qquad \llangle \iota^*\omega, \iota^*\eta\rrangle_M
\end{equation} when $P \subset \R^D$ is sampled from an embedded submanifold $\iota : M \hookrightarrow \R^D$. Appendix~\ref{appendix:B} contains a detailed discussion, from which we extract the following. 

\begin{informalthm}{Continuum Limits}{main-text-continuum-limits}
    Let $M \subset_{\iota} \R^D$ be a manifold with probability density \( q \) from which point clouds $P_{n}=\{ p_{i} \overset{\text{i.i.d.}}{\sim} q \}_{i=1}^n$ are sampled. For each $P_n$, construct the local inner product \( \langle -,-  \rangle_{P_{n}}  \)~\ref{def:gramfield-innerproduct} using the variable-bandwidth Laplacian \( L_{P} \) as described in Appendix~\ref{appendix:B}.  Let $\omega, \eta \in \Omega^k(\mathbb{R}^{D})$. 
  
    \begin{enumerate}
        \item \textit{(Local)}  For any $p \in M$,
    \begin{equation}
        \lim_{n \to \infty} \langle\, \omega, \eta \, \rangle_{P_n}(p) = \langle \iota^*\omega, \iota^*\eta\rangle_M(p) \text{ almost surely},
        \label{eqn:main-text-local-continuum}
    \end{equation} and this convergence is uniform in \( p \) provided \( M \) is compact.
        \item \textit{(Global)} If \( M \) is compact, then
    \begin{equation}
        \lim_{n \to \infty} \llangle\, \omega, \eta \, \rrangle_{P_n} = \llangle \iota^*\omega, \iota^*\eta\rrangle_M \text{ almost surely},
        \label{eqn:main-text-global-continuum}
    \end{equation} where the summation~\eqref{def:gramfield-global-inner-product} is against a standard choice of \( \mu \in \mathbb{R}^{P} \).
    
    \end{enumerate}
   
\end{informalthm} 

\begin{proof}[Proof (Sketch)]

In order to prove Theorem~\ref{ithm:main-text-continuum-limits}, we first provide an expansion on the observation in~\cite{jones2026manifolddiffusiongeometrycurvature}, itself based on \cite{Berry2016}, that the discrete Gram field $G^{(1)}_{P}$ consistently estimates the continuous Gram field $G^{(1)}$. Specifically, we
\begin{enumerate}
    \item analyze how the variable-bandwidth Laplacian parameters affect convergence rates in low-sampling density regions likely to occur in practice (Corollary~\ref{cor:cdc-consistency});
    \item explicate exponential probability bounds under which consistency holds uniformly in the \( (i,j) \)-alignment fields comprising \( G^{(1)} \)(Lemma~\ref{lemma:uniformizing-in-f}); and
    \item explicate exponential probability bounds under which consistency holds uniformly in $p$ when $M$ is compact (Lemma \ref{lemma:uniformizing-in-i-with-compactness}). 
\end{enumerate}
These additional pieces allow us to push consistency estimates on $G^{(1)}$ (Proposition~\ref{prop:tangential-projector-consistency}) uniformly through the determinantal structure~\eqref{eqn:IJ-alignment-field}  in the \( (I,J) \)-alignment fields defining $G^{(k)}$ (Proposition~\ref{cor:kth-tangential-projector-consistency} and Corollary~\ref{cor:kth-tangential-projector-compactness-corollary}), and in turn the local inner product~\eqref{eqn:main-local-inner-product} (Theorem \ref{thm:locla-consistency-estimate}), from which we are able to extract the local continuum limit \eqref{eqn:main-text-local-continuum} almost surely by way of the probability bounds' exponentiality (Corollary~\ref{cor:local-continuum-limit}). Using $(3)$, we further deduce this convergence to be uniform in $p$ when $M$ is compact, and this underlies our further extraction of the global continuum limit~\eqref{eqn:main-text-global-continuum} (Corollary~\ref{cor:global-continuum-limit}). 
\end{proof}
 A feature of our theory is its faithfulness to practice. Our results hold, up to $k$-NN truncation, for the flexible class of variable-bandwidth Laplacians $\Delta_{P}$ used in our implementation. Moreover, they include explicit finite-sample dependence upon (and provable insensitivity to) estimation parameters for $\Delta_{P}$. All of our results also accommodate nonuniform sampling; in the local case \textit{arbitrarily} nonuniform from an unknown density. 

\paragraph{Neural $k$-forms.} The key observation is that ambient differential forms are completely characterized by their scaling functions. This means that differential forms can be parametrized as neural networks and back-propagated over learning tasks \cite{maggs2023simplicial}. Let $\Theta$ denote the parameter space, and let
\begin{equation}
F_{\theta}:\mathbb{R}^D
\to
\mathbb{R}^{\ell \times \binom{D}{k}}
\end{equation}
be a neural network with parameters $\theta \in \Theta$. For each $a\in\{1,\ldots,\ell\}$, the $a$th row of
$F_\theta$ defines the scaling functions of an ambient $k$-form
\begin{equation}
\omega_{\theta}^{(a)}
=
\sum_I
F_{\theta}^{aI} \cdot dx^I
\in \Omega^k(\mathbb{R}^D),
\end{equation}
parametrized by $\theta$. 

\paragraph{Comparison matrices} The final step is to produce an interpretable, learnable vectorization of point-cloud data that depends parametrically on a collection of ambient neural $k$-forms. Given a point cloud $P$, the neural $k$-forms produce the learnable \textit{comparison matrix}
\begin{equation}
\left(C_P(\theta)\right)_{ab} := \llangle \omega^{(a)}, \omega^{(b)} \rrangle_P = 
\sum_{p\in P}
\,
F_{\theta}^{a}(p)^\top
G_P^{(k)}(p)
F_{\theta}^{b}(p) \mu_P(p)
\end{equation}
where $F_{\theta}^{a}(p)\in\mathbb{R}^{\binom{D}{k}}$ denotes the $a$th row of
$F_\theta(p)$. The comparison matrix mapping \begin{equation}
P \, \, 
 \mapsto \, \, 
C_P(\theta)
\in \R^{\ell \times \ell}
\end{equation} represents each point cloud as an $(\ell \times \ell)$-matrix and is permutation-invariant in the points of $P$. It's dimension depends on the number of features rather the number of points. It's entries correspond to inner products of global ambient features, and are across samples in the shared ambient feature space.

\paragraph{Architecture} A key point is that the above theory leads to a straightforward and efficient architecture. Fix an ambient space $\R^D$ consisting of the global features, form dimension $k$ and choice of Laplacian for each point cloud. Intialize a neural $k$-form 
\begin{equation} 
F_\theta: \R^D \to \R^{\ell \times \binom{D}{k}}
\end{equation} with $\ell$ features. Choose a density $\mu_P$ for each point cloud $P$ in the dataset.
\begin{enumerate}
    \item (Gram field transform) The first step performs the mapping
    \begin{equation} 
\Big\{ P \subset \R^D \Big\} \hspace{2em} \mapsto \hspace{2em} \Big\{ G_P^{(k)} : P \to  \R^{\binom{D}{k} \times \binom{D}{k}} \Big\}  
\end{equation} that represents each point cloud $P$ in a dataset by it's Gram field, in practice, a tensor of size $P \times \binom{D}{k} \times \binom{D}{k}$. Since this requires the (expensive) computation of compound matrices, this is performed as pre-computation step as in Algorithm \ref{alg:gram-field}.
\item (Comparison matrix pass) The comparison matrix layer is the mapping
 \begin{equation} 
\big(\,G_P^{(k)}\, , F_\theta \, \big)  \hspace{2em} \mapsto \hspace{2em}\left(C_P(\theta)\right)_{ab} := 
\sum_{p\in P}
\,
F_{\theta}^{a}(p)^\top
G_P^{(k)}(p)
F_{\theta}^{b}(p) \mu_P(p)
\end{equation} described in Algorithm  \ref{alg:comparison-matrix}. This layer is in the forward pass -- the parameters can be thus be back-propagated against any downstream learning task. As per the right-hand side above, this simply requires summing the evaluations of neural scaling vectors against the bilinear forms $G_P^{(k)}(p)$ over $p$ for each pair $1 \leq a\leq b \leq \ell$.  
\end{enumerate} In practice, we vectorize $C_P(\theta)$ with various readout layers, for example
by taking its upper triangular entries, and pass the resulting feature vector to a downstream classifier.

\section{Experiments}
\label{sec:experiments}

\begin{figure}[t]
  \centering
  \newlength{\eTenGridW}
  \setlength{\eTenGridW}{\linewidth}
  \begin{minipage}[c]{\eTenGridW}
    \begin{minipage}[c]{\eTenGridW}
      \centering
      \includegraphics[
        width=\linewidth,
        height=0.15\eTenGridW,
        keepaspectratio
      ]{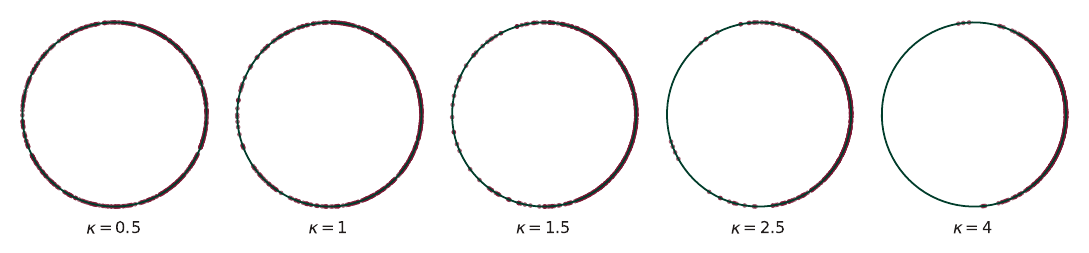}
    \end{minipage}
    \vspace{0.012\eTenGridW}
    \noindent\textcolor{gray!40}{\rule{\linewidth}{0.4pt}}
    \vspace{0.008\eTenGridW}
    \begin{minipage}[c]{0.575\eTenGridW}
      \centering
      \includegraphics[
        width=\linewidth,
        height=0.28\eTenGridW,
        keepaspectratio
      ]{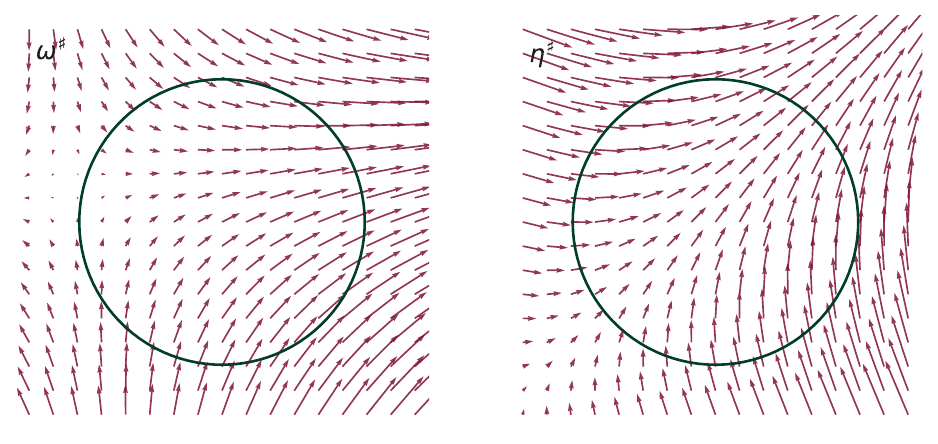}
    \end{minipage}
    \hfill
    \begin{minipage}[c]{0.385\eTenGridW}
      \centering
      \includegraphics[
        width=\linewidth,
        height=0.28\eTenGridW,
        keepaspectratio
      ]{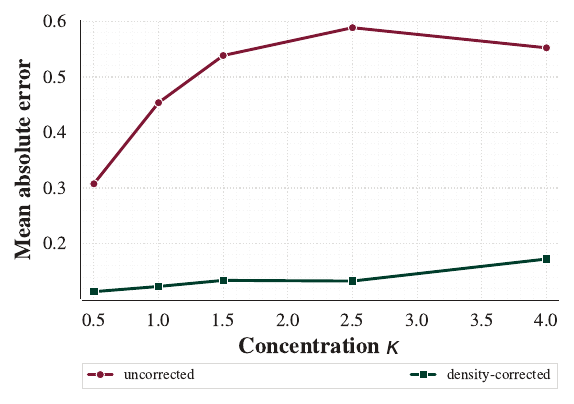}
    \end{minipage}
  \end{minipage}
  \vspace{0.5em}
  \caption{%
    \textbf{Circle one-form density-correction check.}
    \emph{Top:} point clouds sampled from von-Mises densities on \(S^{1}\) as the concentration \(\kappa\) increases.
    \emph{Bottom left:} ambient representatives of the two one-forms used in the inner-product estimator.
    \emph{Bottom right:} MAE under nonuniform sampling shows that density correction substantially reduces the sampling-density bias.
  }
  \label{fig:circle-one-form-density-correction}
\end{figure}

\paragraph{Circles vs lines} We first generated a synthetic dataset consisting of two classes of solution trajectories to the ODEs
\begin{equation}
\text{circles } (\dot{x}, \dot{y}) = (y, -x) \qquad \text{and} \qquad
\text{lines } (\dot{x}, \dot{y}) = (x, y)
\end{equation}
in $\R^2$ over random initial conditions (Figure~\ref{fig:synthetic-form-diagnostics}a). In addition to achieving perfect test-train accuracy, the learned 1-forms appeared to recover the structure of the underlying ODE. Surprisingly, some common GNN benchmarks performed poorly on this relatively straightforward task as Table~\ref{tab:toy_auroc_comparison} shows. We additionally tested non-uniform sampling densities and Gaussian perturbations, finding that variable-bandwidth density correction behaved as expected by our theory (Figure~\ref{fig:circle-one-form-density-correction}).

\paragraph{RNA kinetics}
As a biologically motivated synthetic example, we replaced the synthetic ODEs above with a standard model of RNA transcription dynamics used in RNA velocity \cite{lamanno2018rna}. For each gene \(g\), the dynamics of unspliced and spliced RNA counts are given by
\begin{equation} \label{rna-velocity-eqn}
    \dot{u}_g = \alpha_g - \beta_g u_g, \qquad
    \dot{s}_g = \beta_g u_g - \gamma_g s_g,
\end{equation}
where the features $u_g$ and $s_g$ denote the unspliced and spliced RNA counts, respectively. The parameters $\alpha_g, \beta_g$ and $\gamma_g$ denote the transcription, splicing, and degradation rates.

As a toy model of a biological intervention, we generate one class of flow lines from a fixed parameter set and a second class from a small perturbation of those parameters (Figure~\ref{fig:synthetic-form-diagnostics}b).. The resulting classes correspond to perturbed RNA dynamics. Although the trajectories overlap in feature space, the learned comparison matrix is linearly separable in its first two principal components, showing that our algorithm distinguishes the two classes by their distinct tangency structure (Figure~\ref{fig:synthetic-form-diagnostics}b).

        \begin{figure}[!t]
  \centering

  \newlength{\storyheight}
  \setlength{\storyheight}{0.14\textheight}

  \begin{minipage}[c]{0.35\linewidth}
    \centering
    \includegraphics[
      height=\storyheight,
      keepaspectratio
    ]{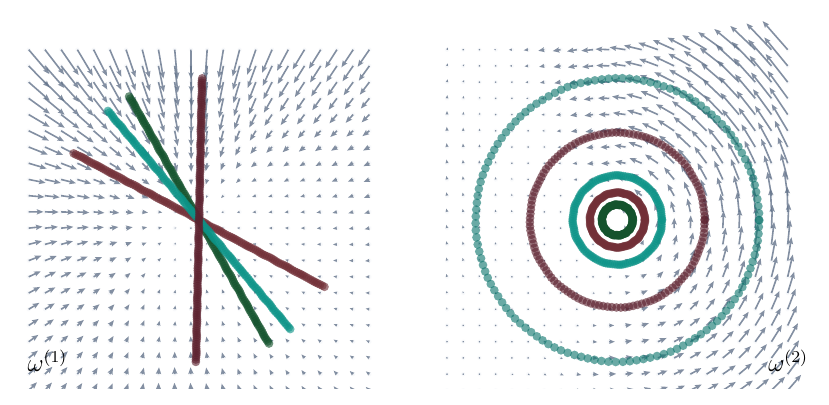}

    {\small\textbf{(a) Circle vs. Lines one-forms}}
  \end{minipage}
  \hfill
  \begin{minipage}[c]{0.6\linewidth}
    \centering
    \includegraphics[
      height=\storyheight,
      keepaspectratio
    ]{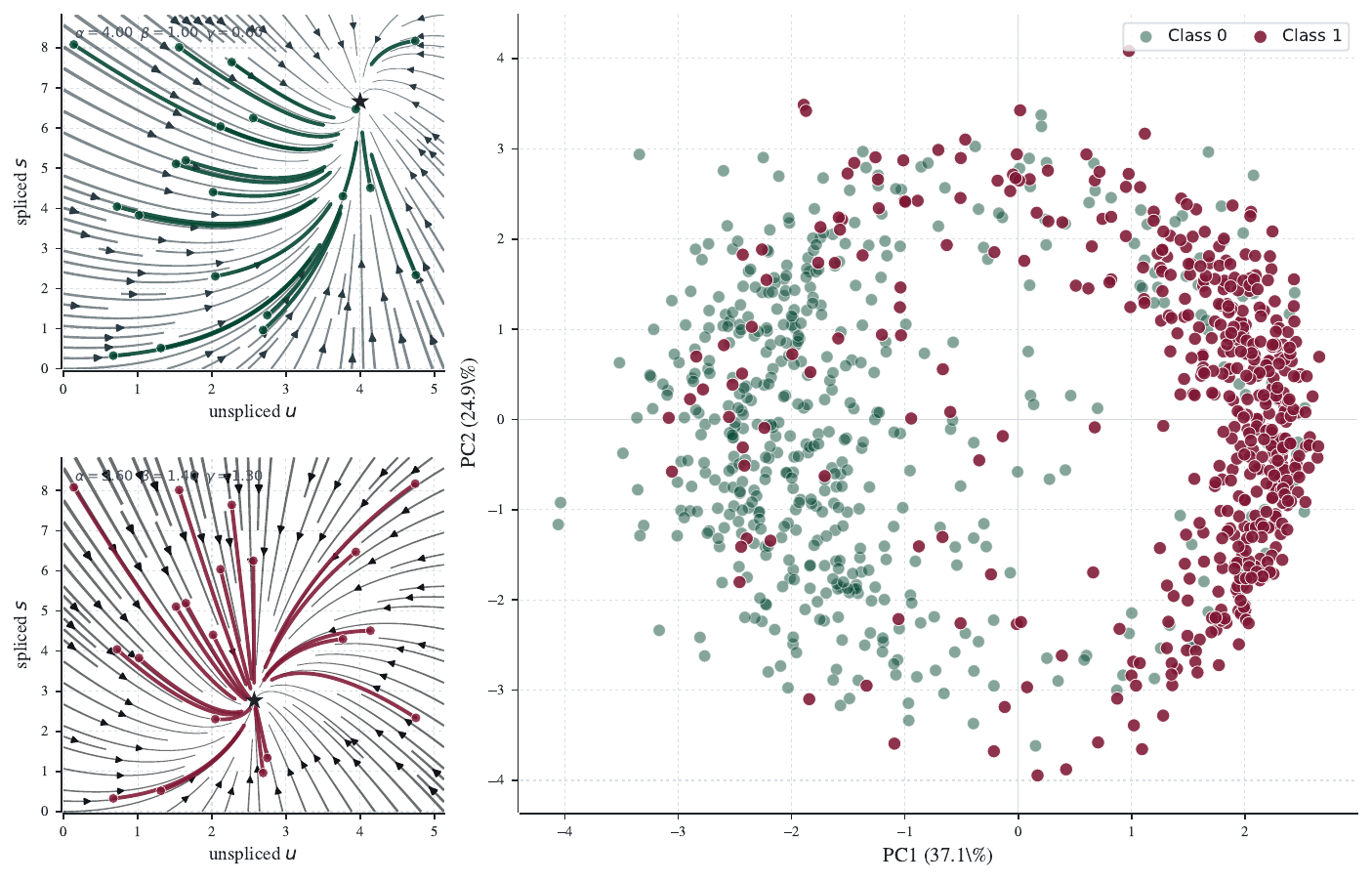}

    \vspace{0.25em}
    {\small\textbf{(b) RNA kinetic regimes}}
  \end{minipage}

  \caption{
  Qualitative diagnostics for learned point-form representations on controlled synthetic tasks.
  \textbf{(a)} Learned one-form channels separate linear and circular structure in the line--circle task.
  \textbf{(b)} Synthetic RNA-kinetic regimes generated by different \((\alpha,\beta,\gamma)\) parameters, with a PCA of learned comparison matrices showing class separation from global inner-product features.
  }
  \label{fig:synthetic-form-diagnostics}
  \end{figure}

   \begin{table*}[t]
            \centering
            \caption{\textbf{Synthetic AUROC comparison.} NPF methods are competitive on both synthetic tasks with at most \(68{,}866\) trainable parameters. NPF-Tri matches the best overall performance on Circles vs.\ Lines, while NPF-Gram is the best NPF variant on synthetic RNA velocity. Gap rows report AUROC differences to the best NPF for each task.}
            \label{tab:toy_auroc_comparison}
            \large
            \setlength{\tabcolsep}{6pt}
            \renewcommand{\arraystretch}{1.25}
            \begin{adjustbox}{max width=\textwidth}
            \begin{tabular}{
                l
                !{\vrule width 1.1pt}
                ccc
                !{\vrule width 1.1pt}
                cccccccc
            }
            \toprule
            &
            \multicolumn{3}{c!{\vrule width 1.1pt}}{\textbf{NPF methods}} &
            \multicolumn{8}{c}{\textbf{Baselines}} \\[4pt]
            \cmidrule(lr){2-4}\cmidrule(lr){5-12}
            \rule{0pt}{2.8ex}\textbf{Dataset / Metric}
            &
            \shortstack{NPF-Tri\\$k=1$}
            & \shortstack{NPF-Gram\\$k=1$}
            & \shortstack{NPF-Flat\\$k=1$}
            & \shortstack{GCN}
            & \shortstack{GIN}
            & \shortstack{Graph\\SAGE}
            & \shortstack{Graph\\Trans.}
            & \shortstack{MLP}
            & \shortstack{Point\\Net++}
            & \shortstack{Point\\Trans.}
            & \shortstack{TDL} \\
            \midrule
            \rowcolor{gray!6}
            \textit{Max params}
            &
            ${\leq}68{,}866$ & ${\leq}68{,}866$ & ${\leq}68{,}866$
            &
            23{,}303 & 23{,}305 & 45{,}447 & 173{,}063 & 31{,}111
            &
            1{,}481{,}735 & 2{,}165{,}767 & 23{,}815 \\[2pt]
            \midrule
            \textbf{Circles vs.\ Lines}
            &
            \bestcell{$1.000{\scriptstyle\pm0.000}$}
            & $0.997{\scriptstyle\pm0.003}$
            & $0.985{\scriptstyle\pm0.028}$
            &
            $0.735{\scriptstyle\pm0.020}$
            & $0.995{\scriptstyle\pm0.005}$
            & $0.726{\scriptstyle\pm0.031}$
            & \underline{$1.000{\scriptstyle\pm0.000}$}
            & $0.729{\scriptstyle\pm0.028}$
            & \underline{$1.000{\scriptstyle\pm0.000}$}
            & \bestcell{$1.000{\scriptstyle\pm0.000}$}
            & $0.731{\scriptstyle\pm0.024}$ \\
            \rowcolor{gray!12}
            \textit{\small Gap to best NPF}
            &
            \zerodelta & \npfdelta{0.003} & \npfdelta{0.015}
            &
            \npfdelta{0.265} & \npfdelta{0.005} & \npfdelta{0.274} & \zerodelta
            & \npfdelta{0.271} & \zerodelta & \zerodelta & \npfdelta{0.269} \\[3pt]
            \midrule
            \textbf{RNA velocity (Syn.)}
            &
            $0.941{\scriptstyle\pm0.012}$
            & \underline{$0.958{\scriptstyle\pm0.014}$}
            & $0.937{\scriptstyle\pm0.017}$
            &
            $0.717{\scriptstyle\pm0.024}$
            & $0.589{\scriptstyle\pm0.045}$
            & $0.725{\scriptstyle\pm0.027}$
            & \underline{$0.973{\scriptstyle\pm0.040}$}
            & $0.728{\scriptstyle\pm0.023}$
            & \underline{$0.958{\scriptstyle\pm0.010}$}
            & \bestcell{$0.982{\scriptstyle\pm0.009}$}
            & $0.861{\scriptstyle\pm0.014}$ \\
            \rowcolor{gray!12}
            \textit{\small Gap to best NPF}
            &
            \npfdelta{0.017} & \zerodelta & \npfdelta{0.021}
            &
            \npfdelta{0.241} & \npfdelta{0.369} & \npfdelta{0.233} & \npfdelta{0.009}
            & \npfdelta{0.230} & \zerodelta & \modeldelta{0.024} & \npfdelta{0.097} \\
            \bottomrule
            \end{tabular}
            \end{adjustbox}
            \end{table*}

        \begin{table}[t]
        \caption{\textbf{PDO response prediction.} AUROC values are reported as mean \(\pm\) confidence interval over completed aggregate runs. Delta columns show AUROC differences to the best selected NPF variant for each task. Positive values favor NPF; negative values favor the row model. Parameter counts are shown where comparable. Selected NPF variants are those best on at least one PDO response task.}
            \label{tab:pdo-main-response}
            \centering
            \tiny
            \setlength{\tabcolsep}{3pt}
            \renewcommand{\arraystretch}{1.04}
            \resizebox{\linewidth}{!}{%
            \begin{tabular}{@{}p{0.14\linewidth}>{\centering\arraybackslash}p{0.14\linewidth}*{3}{>{\centering\arraybackslash}p{0.13\linewidth}G}@{}}
            \toprule
            Model &
            \shortstack{Trainable\\params} &
            \shortstack{PDO response\\tumor-selective\\AUROC \(\uparrow\)} &
            \shortstack{Best NPF\\gap} &
            \shortstack{PDO response\\EFP\\AUROC \(\uparrow\)} &
            \shortstack{Best NPF\\gap} &
            \shortstack{PDO response\\CAF\\AUROC \(\uparrow\)} &
            \shortstack{Best NPF\\gap} \\
            \midrule
            Logistic regression & -- & 0.651 \(\pm\) 0.169 & \npfdelta{0.035} & 0.615 \(\pm\) 0.103 & \npfdelta{0.091} & 0.551 \(\pm\) 0.034 & \npfdelta{0.082} \\
            Random forest & -- & 0.551 \(\pm\) 0.121 & \npfdelta{0.135} & 0.501 \(\pm\) 0.104 & \npfdelta{0.205} & 0.516 \(\pm\) 0.078 & \npfdelta{0.117} \\
            SVM RBF & -- & 0.527 \(\pm\) 0.082 & \npfdelta{0.159} & 0.563 \(\pm\) 0.081 & \npfdelta{0.143} & 0.411 \(\pm\) 0.086 & \npfdelta{0.222} \\
            MLP & 26{,}562--31{,}111 & 0.554 \(\pm\) 0.062 & \npfdelta{0.132} & 0.652 \(\pm\) 0.110 & \npfdelta{0.054} & 0.577 \(\pm\) 0.115 & \npfdelta{0.056} \\
            PointNet & 1{,}465{,}666--1{,}481{,}735 & 0.611 \(\pm\) 0.080 & \npfdelta{0.075} & 0.494 \(\pm\) 0.083 & \npfdelta{0.212} & 0.448 \(\pm\) 0.095 & \npfdelta{0.185} \\
            PointNet++ & 1{,}465{,}666--1{,}481{,}735 & 0.611 \(\pm\) 0.080 & \npfdelta{0.075} & 0.494 \(\pm\) 0.083 & \npfdelta{0.212} & 0.448 \(\pm\) 0.095 & \npfdelta{0.185} \\
            PointTransformer & 2{,}154{,}882--2{,}165{,}767 & 0.587 \(\pm\) 0.067 & \npfdelta{0.099} & 0.558 \(\pm\) 0.117 & \npfdelta{0.148} & 0.518 \(\pm\) 0.074 & \npfdelta{0.115} \\
            GCN & 18{,}434--23{,}303 & 0.493 \(\pm\) 0.154 & \npfdelta{0.193} & 0.534 \(\pm\) 0.121 & \npfdelta{0.172} & 0.474 \(\pm\) 0.159 & \npfdelta{0.159} \\
            GIN & 18{,}436--23{,}305 & 0.618 \(\pm\) 0.124 & \npfdelta{0.068} & 0.685 \(\pm\) 0.116 & \npfdelta{0.021} & \bestcell{\underline{0.636 \(\pm\) 0.074}} & \modeldelta{0.003} \\
            GraphSAGE & 36{,}354--45{,}447 & 0.579 \(\pm\) 0.127 & \npfdelta{0.107} & 0.627 \(\pm\) 0.139 & \npfdelta{0.079} & 0.569 \(\pm\) 0.154 & \npfdelta{0.064} \\
            GraphTransformer & 155{,}522--173{,}063 & 0.544 \(\pm\) 0.112 & \npfdelta{0.142} & \bestcell{\underline{0.710 \(\pm\) 0.068}} & \modeldelta{0.004} & 0.565 \(\pm\) 0.100 & \npfdelta{0.068} \\
            TDL & 18{,}946--23{,}815 & 0.551 \(\pm\) 0.109 & \npfdelta{0.135} & 0.549 \(\pm\) 0.062 & \npfdelta{0.157} & 0.577 \(\pm\) 0.150 & \npfdelta{0.056} \\
            \specialrule{1.2pt}{2.5pt}{2.5pt}
            \multicolumn{8}{l}{\textit{Selected NPF variants: best for at least one reported task}} \\
            \specialrule{0.6pt}{1.2pt}{1.2pt}
            NPF-Tri \(k=1\) & \(\le 68{,}866\) & \underline{0.665 \(\pm\) 0.092} & \npfdelta{0.021} & \underline{0.700 \(\pm\) 0.050} & \npfdelta{0.006} & \underline{0.633 \(\pm\) 0.082} & \zerodelta \\
            NPF-Gram \(k=1\) & \(\le 68{,}866\) & \underline{0.664 \(\pm\) 0.057} & \npfdelta{0.022} & 0.664 \(\pm\) 0.102 & \npfdelta{0.042} & \underline{0.589 \(\pm\) 0.047} & \npfdelta{0.044} \\
            NPF-Flat \(k=1\) & \(\le 68{,}866\) & 0.629 \(\pm\) 0.082 & \npfdelta{0.057} & \underline{0.706 \(\pm\) 0.064} & \zerodelta & \underline{0.589 \(\pm\) 0.084} & \npfdelta{0.044} \\
            NPF-Flat \(k=2\) & \(\le 68{,}866\) & \bestcell{\underline{0.686 \(\pm\) 0.097}} & \zerodelta & \underline{0.704 \(\pm\) 0.148} & \npfdelta{0.002} & 0.581 \(\pm\) 0.073 & \npfdelta{0.052} \\
            \bottomrule
            \end{tabular}%
            }
        \end{table}
\paragraph{Patient-derived organoid response tasks.} For a more realistic real-world task, we adapted the benchmark from \cite{viswanath2025hiponet} on publicly available patient-derived organoids data (PDOs) \cite{RamosZapatero2023, mendeley2022pdo} where the underlying governing response equations are unknown. PDOs are three-dimensional cultures grown from patient tumors and are used to measure treatment response in a setting that preserves aspects of the original tissue. A sample is a condition-level empirical measure \(\mu_a=m_a^{-1}\sum_{\ell=1}^{m_a}\delta_{z_{a\ell}}\), where \(z_{a\ell}\in\mathbb{R}^D\) is the marker vector of one cell in condition \(a\). The label is \(y_a=\mathbf{1}\{S_a>\tau\}\), where \(S_a\)  is a matched response score. Based on interpretable biological markers we report three endpoints \: tumor-selective response, epithelial function perturbation (EFP), which measures treatment-induced changes in tumor epithelial cell state, and cancer-associated fibroblast (CAF)-mediated epithelial reprogramming, which measures how nearby support cells in the tumor microenvironment influence the behavior and identity of tumor cells (See. \ref{subsec:pdo-response-data-construction}). These endpoints are relevant from both the biological and geometric perspectives: they separate direct tumor killing, functional state change, and microenvironment-driven factors and we can detect response geometry in marker space: local direction, marker covariations, and population-level shifts that distinguish treated conditions from controls. 

\paragraph{PDO result pattern.} Table~\ref{tab:pdo-main-response} and support a bounded claim. The selected NPF rows give the highest point estimate for tumor-selective response. Differences are \emph{smaller than} the reported confidence intervals. A compact form-comparison representation is competitive with larger point-cloud and graph models on condition-level response clouds. We also observe that performance does not simply improve by increasing \(k\): tumor-selective endpoint is strongest for NPF-Flat \(k=2\), but EFP and CAF endpoints are strongest among selected NPF rows at \(k=1\). Geometrically, this points to first-order directions and pairwise marker covariants as the stable signal in these PDO tasks, with higer degree volume terms helping in a more endpoint-specific way. This matches the formulation defined in section~\ref{sec:formulation}: \(G_P^{(k)}\) can expose higher-order tangent information, but the data does not require every endpoint to use it. 

Regarding density kernels, synthetic tests treat density as a nuisance because the support relation is fixed while the sampling law changes. In PDO tasks, local cell abundance can itself be part of the response. The practical reading is therefore narrower: variable bandwidth is useful for adapting neighborhood scale, but \textit{removing all density variation can also remove biological signal}.

\section{Discussion}
\paragraph{Experimental takeaway.} The consistency results justify \(G_P^{(k)}\) as a discrete estimator of restricted form pairings under the stated smooth-sampling assumptions. The experiments test whether the induced comparison matrix \(C_P(\theta)\) is useful as a learned point-cloud representation. On controlled systems, including line--circle trajectories and synthetic RNA-kinetic regimes, NPFs identify class structure tied to tangent directions and perturbations of an underlying vector field. On patient-derived organoid response tasks, where the governing biological relation is unknown, the same form-comparison layer remains competitive under a small parameter budget: the reported NPF variants use at most \(68{,}866\) trainable parameters, achieve the strongest point estimate on tumor-selective response, and are statistically comparable with the leading models on EFP and CAF. These results support a bounded empirical claim: compact comparisons of learned ambient forms can encode response-dependent cell-state geometry, including tangent, covariance, and density-related information, without relying on the parameter scale of the largest point-cloud baselines.

\paragraph{Limitations} One limitation is the number of features required to parametrize a neural $k$-form grows factorially in both $k$ and in ambient dimension $D$. This is prohibitive even when applying low-dimensional differential forms in high-dimensional feature spaces. A second limitation is that complex-free geometric methods, including Diffusion Geometry, tend to scale poorly to higher dimensional datasets due to the curse of dimensionality.

\paragraph{Scaling properties} In this paper, we introduce an entirely novel architecture. Thus, our primary focus was on small synthetic and real-world datasets that illustrate the basic properties of the model. Competing message-passing approaches such as GNNs and TDL are known to suffer from over-smoothing \citep{li2018deeper,oono2019graph,cai2020note,rusch2023survey} and over-squashing \citep{alon2021bottleneck,topping2022understanding,di2023over}, making their utility limited on very large datasets. Given our architecture relies on a radically different foundation, one future work is to apply neural point forms on significantly bigger datasets to explore whether their scaling behaviour may overcome some of the limitations of other models. 

\paragraph{Interpretability} An important point is that our extrinsic architecture only makes sense in the context of a global feature space. In such situations, the extrinsic features have an application specific meaning. For example, the features in our biological example and in RNA sequencing correspond to concentrations of certain biological molecules. As discussed above, the post-trained ambient $k$-forms $F_\theta$ encode which $k$-dimensional feature subspaces were most salient for the learning task. A direction of future work is to design methods whereby the learned forms may be interpreted, for example, to understand and describe higher order combinations of features are relevant to applications.


\paragraph{Beyond.} Beyond the experiments in this paper, Neural Point-Forms suggest a broader route for using geometry inside modern ML systems, including agentic pipelines that must store, retrieve, and act over structure state spaces \citep{ha2018worldmodels,yao2023react,park2023generative}. The featurization presented in this project could complement standard point-set and embedding methods by exposing local tangent structure, density, orientation, and multiscale organization, which are precisely the kinds of signals lost in purely nearest-neighbor retrieval \citep{qi2017pointnet,bronstein2021geometric,coifman2006diffusion}. Future work could also consider whether these geometric summaries improve memory, planning, and intervention modules in domains like cell states, proteins, materials, robotics, and scientific simulations \citep{hirani2003discrete,desbrun2003discrete,lewis2020retrieval}.




{
\small

\bibliographystyle{unsrtnat}
\bibliography{references}
}

\newpage
\appendix

\section*{Appendices}




\tableofcontents

\section{Riemannian Metrics, Differential Forms and $L^2$-inner Products}
\label{app:background-diff-forms}
\label{appendix:A-smooth-review}

In this section, we provide a rigorous exposition of differential forms, pullbacks and their $L^2$-inner products induced by Riemannian metrics. In particular, we justify the expository definitions provided in the main text. A more comprehensive overview can be found in either \cite{Lee2018} or \cite{Jost2017}.

\paragraph{Gram matrices} Let $V$ be an inner product space with inner product $\langle-, -\rangle$. The \textit{Gram matrix} of a tuple $(v_1, \ldots, v_\ell)$ of vectors $v_i \in V$ is the matrix
\begin{equation}
    G(\mathbf{v}) = \Big[ \langle v_i,v_j \rangle \Big]_{i,j} \in \R^{\ell \times \ell}.
\end{equation}

\paragraph{Differential forms} A concise and unhelpful definition of a \textit{differential $k$-form} $\omega \in \Omega^k(M)$ on a manifold $M$ is as a smooth section of the $k$-th exterior algebra of the cotangent bundle. Equivalently, $\omega \in \Omega^k(M)$ specifies a family of alternating functions
\begin{equation}
    \omega_p : T_pM \times \ldots \times T_pM \to \R
\end{equation} parametrized over points $p \in M$. This function $\omega_p$ takes as input $k$  tangent vectors $(v_1, \ldots, v_k)$ at $p$ and outputs a real number, loosely interpreted as volume spanned by $(v_1, \ldots, v_k)$ with respect to $\omega$.

\paragraph{Ambient forms} In Euclidean space, differential $k$-forms exhibit a convenient global structure. Explicitly, for coordinate functions $(x^1, \ldots, x^D)$ on $\R^D$, the differential $k$-forms in $\R^D$ satisfy the following isomorphism
\begin{equation}
    \Omega^k(\R^D) \cong C^\infty(\R^D) \otimes \Lambda^k(dx_1, \ldots, dx_D)
\end{equation}  where the tensor product is over $\R$ and $\Lambda^k(dx^1, \ldots, dx^D)$ is the $k$-th exterior algebra over the symbols $dx^1, \ldots, dx^D$. In other words, this means that every differential $k$-form in $\R^D$ can be written in a canonical form\footnote{Employing Einstein notation.}
\begin{equation}
    \omega = f_I \cdot dx^I := \sum_I f_I \cdot dx^I
\end{equation} for some smooth \textit{scaling function} $f_I \in C^\infty(\R^D)$ and basis form $dx_I \in \Lambda^k(dx^1, \ldots, dx^D)$, where $I$ ranges over multi-indices $I = (1\leq i_1 < \cdots < i_k \leq D)$.

\paragraph{Pullbacks} Suppose we have a smoothly embedded manifold $\iota : M \hookrightarrow \R^D.$ Since $\iota$ is an embedding, there are natural inclusions and orthogonal projections of tangent spaces
\begin{equation}
    T_p M \hookrightarrow T_p \R^D \xrightarrow{\Pi_p} T_p M
\end{equation} at each point $p \in M$. Let $\omega \in \Omega^k(\R^D)$. The \textit{pullback form} $\iota^*\omega \in \Omega^k(M)$ of $\omega$ along $\iota$ is the differential $k$-form defined by the family of alternating functions
\begin{equation}
    \iota^*\omega_p : T_pM \times \ldots \times T_pM \to \R; \hspace{2em} \iota^*\omega_p(v_1, \ldots, v_k) = \omega_{\iota(p)}(v_1, \ldots, v_k)
\end{equation} where we have identified $v_p \in T_pM$ with its inclusion into $T_p\R^D$.
\begin{figure}[H]
    \centering
    \begin{subfigure}[t]{0.4\textwidth}
        \centering
        \includegraphics[width=\linewidth, height=0.7\textheight, keepaspectratio]{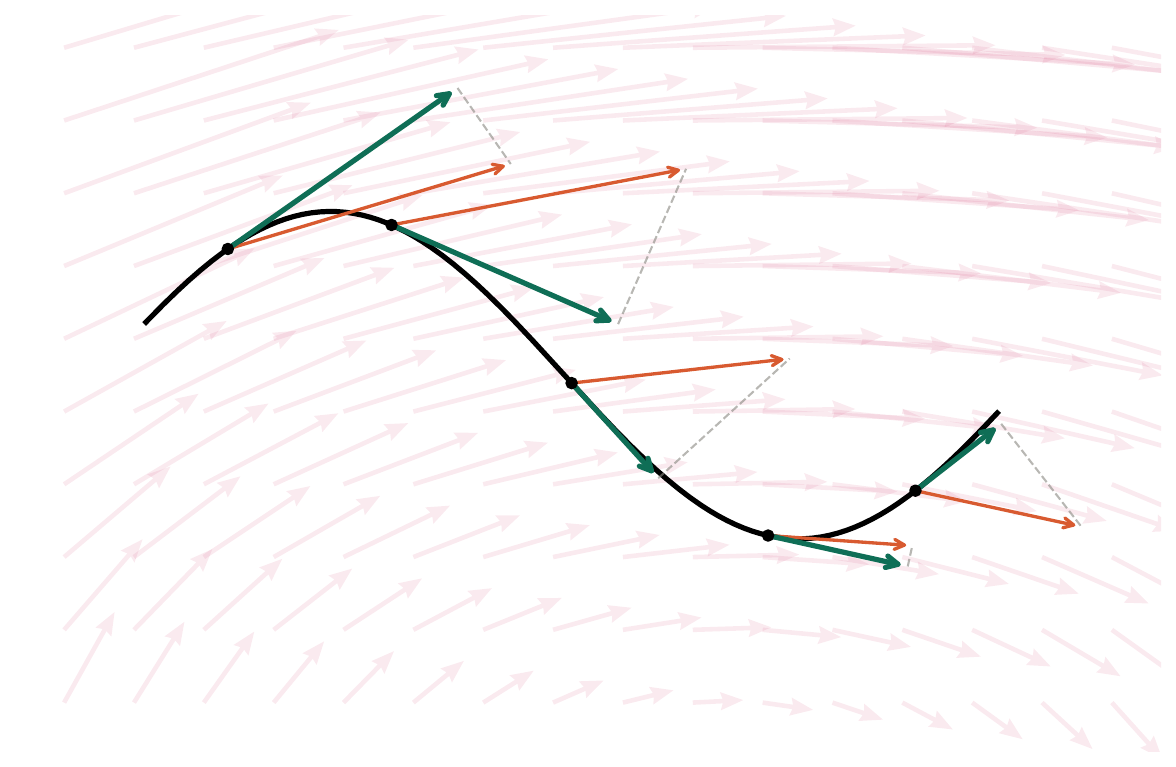}
        \caption{\(k=1\)}
        \label{subfig:k1}
    \end{subfigure}
    \hfill
    \begin{subfigure}[t]{0.4\textwidth}
        \centering
        \includegraphics[width=\linewidth, height=0.7\textheight, keepaspectratio]{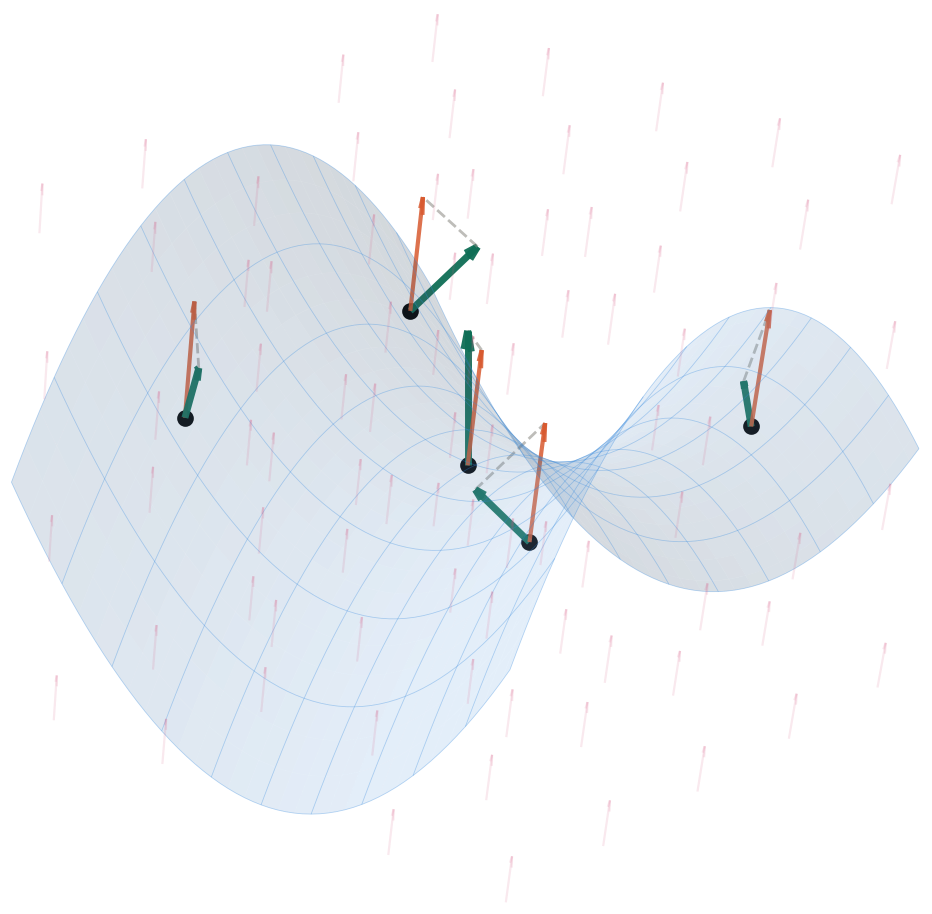}
        \caption{\(k=2,\ D=3\)}
        \label{subfig:k2}
    \end{subfigure}
    \caption{\( \ref{subfig:k1}\): The pullback of a \( 1 \)-form on \( \mathbb{R}^{2} \) (viewed as a vector field) to a \( 1 \)-manifold. \( \ref{subfig:k2}\): The pullback of a \( 2 \)-form on \( \mathbb{R}^{3} \) (viewed as a normal field) to a \( 2 \)-manifold.}
    \label{fig:restricted-forms}
\end{figure}

\paragraph{Riemannian metrics} A Riemannian metric $g$ on a manifold $M$ corresponds to a family of inner products
\begin{equation}\label{eq:tangent-innerprod}
    \langle -, - \rangle_M(p): T_p M \otimes T_pM \to \R
\end{equation} parametrized by points $p \in M$. The inner product induces an isomorphism
\begin{equation} 
    \flat : T_{p}M \underset{\sharp}{\overset{\flat}{\rightleftarrows}} T_{p}^{*}M  : \ \sharp
\end{equation} between tangent and cotangent space for all $p$, and specifies a family of inner products
\begin{equation} \label{eq:cotangent-innerprod}
    \langle -, - \rangle_M(p) :T_p^* M \otimes T_p^*M \to \R
\end{equation} where we have abused notation by conflating \ref{eq:tangent-innerprod} and \ref{eq:cotangent-innerprod}. 

\paragraph{Smooth \textit{Carré-du-champ}} Let $\Delta_M : C^\infty(M) \to C^\infty(M)$ be the Laplace-Beltrami operator on smooth functions. The \textit{Carré-du-champ} identity
\begin{equation}
    \big\langle \,df, dh \,\big\rangle_{M}(p) = \Gamma_{\Delta_M}(f,h)(p) :=\dfrac{1}{2} \Big( f \Delta_M h + h \Delta_M f - \Delta_M(fh) \Big)(p)
\end{equation} relates the inner product on the differentials $df,dh \in \Omega^1(M)$ of smooth functions $f,h \in C^\infty(M)$ with the Laplace-Beltrami operator and the point-wise multiplication of functions.

\paragraph{Point-wise inner products on forms} For two differential $1$-forms $\alpha,\beta \in \Omega^1(M)$, the point-wise inner product $\langle \alpha, \beta \rangle_M(p)$ is directly defined in  Equation \ref{eq:cotangent-innerprod}. In local coordinates $(y^1, \ldots, y^d)$, $k$-forms $\alpha$ and $\beta$ can be written as
\begin{equation}
    \alpha = f_I \cdot dy^I \qquad \beta = h_J \cdot dy^J
\end{equation} where $I = (1\leq i_1 < \ldots < i_k \leq d)$. The inner product of $\alpha$ and $\beta$ at $p$ is then given by
\begin{equation}
    \langle \, \alpha, \beta \, \rangle_M(p) = f_I(p) h_J(p) \det_{a,b} \Big[ \, \langle \,dy^{i_a}, dy^{i_b} \,\rangle_M(p) \Big].
\end{equation}

\paragraph{Gram fields} Suppose we have two ambient $k$-forms
\begin{equation}
    \omega = f_I \cdot dx^I \qquad \text{and} \qquad \eta = h_J \cdot dx^J.
\end{equation} The pullback $\iota^*$ commutes with multiplication of basis forms by scaling functions. In conjunction with the pointwise bilinearity of the inner product, this shows that
\begin{equation}
    \langle\, \iota^* \omega, \iota^*\eta\, \rangle_M(p) = \Big\langle f_I \cdot \iota^*dx^I, h_J \cdot \iota^* dx^J \, \Big \rangle_M(p) = \Big( \,f^T \, G^{(k)}_M \, h\, \Big)(p)
\end{equation} where we have conflated $f_I$ and $h_J$ with their restrictions $\iota^* f_I$ and $\iota^* h_J$. This shows that the expression on the right-hand side, which was given as the \textit{definition} of the inner product in the main text \ref{def:gramfield-innerproduct}, is consistent with the standard, intrinsic inner product of forms.

\paragraph{Gradient interpretation} Let 
\begin{equation} 
\nabla_M : C^\infty(M) \to \mathfrak{X}(M)
\end{equation}be the gradient operator on $M$ and 
\begin{equation} 
\nabla : C^\infty(\R^D) \to \mathfrak{X}(\R^D)
\end{equation} be the ambient Euclidean gradient. When $\iota : M \hookrightarrow \R^D$ is a submanifold, the two gradient operators are related by
\begin{equation} \label{eq:extrinsic-exterior-derivative}
    \nabla_M \iota^*f(p) = \Pi_p \nabla f(p)
\end{equation} for all $p$. In the main text, we \textit{defined} the inner product on basis $1$-forms to satisfy
\begin{equation} \label{eq:pullback-projection-equiv}
    \Big\langle \iota ^*dx^i, \iota^*dx^j \Big\rangle_M(p) = \Big\langle \Pi_p \nabla x^i, \Pi_p \nabla x^j\Big\rangle_{\R^D}(p).
\end{equation} To see this is equivalent, recall that the Riemannian metric identifies
\begin{equation}
    \langle \, df,dh \, \rangle_M(p ) = \langle \, \nabla_M f, \nabla_Mh \, \rangle_M(p)
\end{equation} for all $f,h \in C^\infty(M)$ and $p \in M$. Combining this with \ref{eq:extrinsic-exterior-derivative} and using the fact that the exterior derivative commutes with pullbacks, we get
\begin{equation}
    \Big\langle \iota ^*dx^i, \iota^*dx^j \Big\rangle_M(p) = \Big\langle d\iota ^*x^i, d\iota^*x^j \Big\rangle_M(p) = \Big\langle \Pi_p \nabla x^i, \Pi_p \nabla x^j\Big\rangle_{\R^D}(p).
\end{equation} as claimed.

 \paragraph{Normal-tangent decompositions} In \ref{eq:gramfield-decomposition} we claimed that the tangent and normal spaces were
\begin{equation}
    \text{Im} \,G_M(p) = T_pM \qquad \text{Ker} \, G_M(p) = N_pM.
\end{equation} To see this, note that \ref{eq:pullback-projection-equiv} implies that the entries of the Gram field at $p$ are
\begin{equation}
    G_M(p)_{i,j} = \Big\langle \Pi_p \nabla x^i, \Pi_p \nabla x^j\Big\rangle_{\R^D}(p) = \Big\langle  \nabla x^i, \Pi_p \nabla x^j\Big\rangle_{\R^D}(p),
\end{equation} where the second equation follows from the fact that the orthogonal projection operator $\Pi_p$ is self-adjoint and idempotent $\Pi_p^2 = \Pi_p$. This shows that the Gram field $G_M(p)$ is the matrix representation of $\Pi_p$ in the standard directional coordinate basis $\nabla x^i$, with the decomposition thus following immediately. For $k > 1$ and $p \in M$, the general decomposition
\begin{equation} \label{eq:k-normal-tangent-decomp}
    \text{Im}\,G_M^{(k)}(p) = \Lambda^k T_pM \qquad \quad \text{Ker} \, G_M^{(k)}(p) = \big(\,\Lambda^kT_pM \,\big)^\perp
\end{equation}follows from the fact that $G_M^{(k)}$ is the matrix form of the induced operator
\begin{equation}
    \Lambda^k \Pi_p^* : \Lambda^k T^*_p \R^D \to  \Lambda^kT^*_pM
\end{equation} over the wedge basis and, thus, also an orthogonal projection.

\paragraph{$L^2$ inner products}
The pointwise pairing of $k$-forms defines a real-valued function on $M$.
For $\alpha,\beta \in \Omega^k(M)$, we write this function as
\begin{equation}
p \mapsto \langle \alpha,\beta\rangle_{M}(p).
\end{equation}
A global comparison is obtained by integrating this pointwise pairing over the manifold. Given a measure $\mu$ on $M$, define
\begin{equation}
\left\langle\!\left\langle
\alpha,\beta
\right\rangle\!\right\rangle_{L^2(M,\mu)}
:=
\int_M
\left\langle
\alpha,\beta
\right\rangle_{M}(p)
\,d\mu(p).
\end{equation}
We assume $\alpha$ and $\beta$ are compactly supported to avoid issues with the integral being undefined. In the unweighted case, $\mu=dV_M$ is the Riemannian volume measure. If $q:M \to \R_{\geq 0}$ is a sampling density, we may instead take
$\mu=q\,dV_M$. Since our forms arise by restricting ambient forms to $M$, for
$\omega,\eta\in\Omega^k(\mathbb{R}^D)$ we define
\begin{equation}
\left\langle\!\left\langle
\iota^*\omega,\iota^*\eta
\right\rangle\!\right\rangle_{L^2(M,\mu)}
=
\int_M
\left\langle
\iota^*\omega,\iota^*\eta
\right\rangle_{M}(p)
\,d\mu(p).
\end{equation}

\newpage
\section{Algorithms}
In this section we present the basic pseudocode for the two main algorithms. In our implementation of Algorithm \ref{alg:gram-field}, the Laplacian we construct is always the variable bandwidth Laplacian specified in \ref{discrete-variable-bandwidth-laplacian}.

\begin{algorithm}[H]
\caption{Gram field precomputation}
\label{alg:gram-field}
\begin{algorithmic}[1]
\Require Point cloud $P=\{p_1,\ldots,p_n\}\subset\mathbb{R}^D$, degree $k$
\Ensure $G_P^{(k)}:P\to\operatorname{Sym}_{\binom{D}{k}}$

\State Construct a discrete Laplacian $L_P:\mathbb{R}^P\to\mathbb{R}^P$.
\State Let $x^i\in\mathbb{R}^P$ be the $i$th coordinate function on $P$.

\For{$i,j\in\{1,\ldots,D\}$}
    \State $G_P^{ij}\gets
    \frac{1}{2}\left(
    x^i\odot L_Px^j+x^j\odot L_Px^i-L_P(x^i\odot x^j)
    \right)$
\EndFor

\State Let $\mathcal{I}_k=\{(i_1<\cdots<i_k)\subset\{1,\ldots,D\}\}$.

\For{$p\in P$}
    \For{$I,J\in\mathcal{I}_k$}
        \State $\left(G_P^{(k)}(p)\right)^{IJ}
        \gets
        \det\left[G_P^{i_a j_b}(p)\right]_{a,b=1}^k$
    \EndFor
\EndFor

\State \Return $G_P^{(k)}$
\end{algorithmic}
\end{algorithm}

\begin{algorithm}[H]
\caption{Comparison matrix forward pass}
\label{alg:comparison-matrix}
\begin{algorithmic}[1]
\Require Point cloud $P$, Gram field $G_P^{(k)}$, weights $\mu_P$, neural $k$-form $F_\theta:\mathbb{R}^D\to
\mathbb{R}^{\ell\times\binom{D}{k}}$
\Ensure Comparison matrix $C_P(\theta)\in\operatorname{Sym}_{\ell}$

\State $C_P(\theta)\gets 0_{\ell\times\ell}$

\For{$p\in P$}
    \State $F_p\gets F_\theta(p)$
    \State $C_P(\theta)\gets C_P(\theta)
    +\mu_P(p)\,F_pG_P^{(k)}(p)F_p^\top$
\EndFor

\State \Return $C_P(\theta)$
\end{algorithmic}
\end{algorithm}

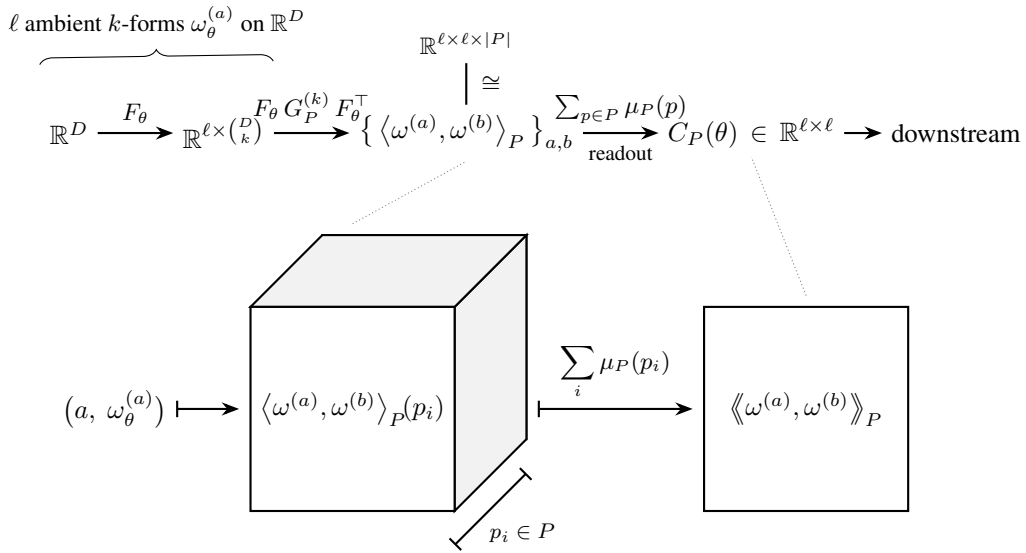
\begin{figure}[H]
    \centering
    \begin{tikzpicture}[
    >=Stealth,
    arr/.style={->, thick},
    mapsto/.style={|->, thick},
    lbl/.style={font=\small, midway, above},
]


\node (RD) {$\mathbb{R}^{D}$};

\node[right=1.0cm of RD] (FRange) {$\mathbb{R}^{\ell \times \binom{D}{k}}$};
\draw[arr] (RD) -- node[lbl] {$F_\theta$} (FRange);

\draw[decorate, decoration={brace, amplitude=5pt, raise=18pt}]
    (RD.north west) -- (FRange.north east)
    node[midway, above=24pt, font=\small]
    {$\ell$ ambient $k$-forms $\omega^{(a)}_\theta$ on $\mathbb{R}^{D}$};

\node[right=1.0cm of FRange] (PerP)
    {$\bigl\{\,\bigl\langle \omega^{(a)}, \omega^{(b)}\bigr\rangle_{P}\,\bigr\}_{a,b}$};
\draw[arr] (FRange) -- node[lbl] {$F_\theta\, G^{(k)}_P\, F_\theta^{\top}$} (PerP);

\node[above=0.55cm of PerP, font=\small] (Shape)
    {$\mathbb{R}^{\ell \times \ell \times |P|}$};
\draw[thick] (PerP) -- node[midway, right=2pt, font=\small] {$\cong$} (Shape);

\node[right=1.0cm of PerP] (CP) {$C_P(\theta)\,\in\,\mathbb{R}^{\ell \times \ell}$};
\draw[arr] (PerP) -- node[lbl] {$\sum_{p\in P}\mu_P(p)$}
    node[below=1pt, midway, font=\footnotesize] {readout} (CP);

\node[right=0.5cm of CP] (task) {downstream};
\draw[arr] (CP) -- (task);


\def\cx{2.4} \def\cy{-5.0}
\def\cs{2.7}
\def\cd{0.95}

\node at (\cx - 1.8, \cy+0.5*\cs) (input) {$\bigl(a,\; \omega^{(a)}_\theta\bigr)$};

\draw[thick, fill=white] (\cx+\cd, \cy+\cd) rectangle +(\cs, \cs);
\draw[thick, fill=gray!10]
    (\cx+\cs, \cy) -- +(\cd, \cd) -- +({\cd}, {\cd+\cs}) -- +(0, \cs) -- cycle;
\draw[thick, fill=gray!10]
    (\cx, \cy+\cs) -- +(\cd, \cd) -- +({\cd+\cs}, \cd) -- +(\cs, 0) -- cycle;
\draw[thick, fill=white] (\cx, \cy) rectangle +(\cs, \cs);

\node[font=\normalsize] at (\cx+0.5*\cs, \cy+0.5*\cs)
    {$\bigl\langle \omega^{(a)}, \omega^{(b)} \bigr\rangle_{P}\!(p_i)$};

\coordinate (pLineL) at ($({\cx+\cs}, {\cy}) + (0, -0.45)$);
\coordinate (pLineR) at ($({\cx+\cs+\cd}, {\cy+\cd}) + (0, -0.45)$);
\draw[thick, {|}-{|}] (pLineL) -- (pLineR);
\node[font=\footnotesize, below=2pt, xshift=12pt] at ($(pLineL)!0.5!(pLineR)$) {$p_i \in P$};

\draw[mapsto] (input.east) -- (\cx-0.15, \cy+0.5*\cs);

\def\mx{8.4} \def\my{-5.0}
\draw[thick, fill=white] (\mx, \my) rectangle +(\cs, \cs);
\node[font=\normalsize] at (\mx+0.5*\cs, \my+0.5*\cs)
    {$\bigl\langle\!\!\bigl\langle \omega^{(a)}, \omega^{(b)} \bigr\rangle\!\!\bigr\rangle_{P}$};

\draw[mapsto] (\cx+\cs+\cd+0.15, \cy+0.5*\cs) --
    node[midway, above, font=\small] {$\displaystyle\sum_{i} \mu_P(p_i)$}
    (\mx-0.15, \my+0.5*\cs);

\draw[gray, densely dotted] (\cx+0.5*\cs, \cy+\cs+\cd+0.15) -- (PerP.south);
\draw[gray, densely dotted] (\mx+0.5*\cs, \my+\cs+0.15) -- (CP.south);

\end{tikzpicture}
    \caption{NPF Learning Pipeline.}
    \label{fig:placeholder}
\end{figure}



\section{Discretization}
\label{appendix:B}

The purpose of this appendix section is to investigate in detail the theoretical and practical groundwork for neural point-forms. In Section~\ref{Appendix:subsec:discrete-kfcs} we further discuss the relevant definitions and implementation details used in our experiments, and in Section~\ref{Appendix:Subsec:theory} we prove consistency results justifying these definitions and practices. Future appendices apply this groundwork in experiments.

Throughout this section, $M \subset \mathbb{R}^{D}$ is a \( d \)-manifold with nonvanishing probability density $q \in C^3(M) \cap L^1(M)$, from which a point cloud $P=\{ p_{i} \overset{\text{i.i.d.}}{\sim} q \}_{i=1}^n$ is sampled.

\subsection{Point Form Estimation}
\label{Appendix:subsec:discrete-kfcs}
\label{appendix:B1}

In this subsection, we provide details regarding the definition and implementation of point forms on data. As seen in Appendix~\ref{appendix:A-smooth-review} and the main text, the continuous pipeline relies on three components: the metric \( g \) with which to locally compare \( 1 \)-forms, a formalism extending this construction to the local comparison of general \( k \)-forms, and a measure \( \mu \) against which one may average local comparisons to obtain a global comparison. We elaborate upon the discretization of each notion in turn.

\subsubsection{Estimating the Carré du Champ From Data}
\label{subsubsec:cdc-estimation-from-data-including-variable-bandwidth-laplacian-contruction}

In this section we assume $M$ potentially noncompact, as we wish to accommodate the possibility that the sampling density $q>0$ can become arbitrarily close to zero. 

\paragraph{Variable-Bandwidth Diffusion Kernels and Laplacians}

One approach to estimating the carré du champ $\Gamma$ on $P$ is to directly estimate a diffusion generator $\Delta$ on \( M \) and from it mimic Equation~\ref{eqn:one-of-many-cdc-identities}.\footnote{There are similar but distinct approaches which bypass intermediate Laplacian estimation, see e.g.~\cite{bamberger2025carr, jonesComputingDiffusionGeometry2026}.} This is generally accomplished by defining a particular transition kernel on the data and examining the ``discrete generator" of a corresponding Markov chain. Our practice and theory both work with the flexible class of \textbf{variable-bandwidth diffusion kernels}, following the construction in~\cite{Berry2016} given by the following sequence of operations:
\begin{enumerate}
    \item Select $\beta < 0.$
    \item Define a ``bandwidth function" \( \rho:M \to \mathbb{R}_{+} \). In the ideal setting, one takes
    \begin{equation}
        \rho(p_i) := q(p_i)^\beta 
    \end{equation} using the true sampling density. 
    \item Select a global bandwidth parameter $\varepsilon > 0$. 
    \item Define the \textit{unnormalized} variable bandwidth kernel 
    \begin{equation}
       K_{\varepsilon}(x,y)=h\left(\frac{\|x-y\|^{2}}{\varepsilon \rho(x) \rho(y)}\right)     \end{equation}for some $h$ with exponential decay. By default, we take $h(u):=e^{-u/4}$.

    \item Define the unnormalized density estimate
    \begin{equation}
        q_\varepsilon(p_i) = \sum_{j=1}^N \dfrac{K_\varepsilon^{}(p_i,p_j)}{\rho(p_i)^d}.
    \end{equation}
    \item Select a density sensitivity parameter $\alpha > 0.$
    \item Define the symmetric \textit{$\alpha$-density weighted kernel} 
    \begin{equation}
        K_{\varepsilon,\alpha}^{}(p_i, p_j) := \dfrac{{K_\varepsilon}^{}(p_i,p_j)}{q_\varepsilon^{}(p_i)^\alpha q_\varepsilon^{}(p_j)^\alpha}.
    \end{equation}
    \item In order to normalize into a Markov (probability) kernel, first take the row sums
    \begin{equation}
        q_{\varepsilon,\alpha}^{}(p_i) := \sum_{j=1}^N K_{\varepsilon,\alpha}^{}(p_i,p_j),
    \end{equation}
    \item and then define 
    \begin{equation}
        \hat{K}^{}_{\varepsilon,\alpha}(p_i,p_j) := \dfrac{K_{\varepsilon,\alpha}^{}(p_i,p_j)}{q_{\varepsilon,\alpha}^{}(p_i)}.
    \end{equation}
    \item Define the ``discrete diffusion generator"
    \begin{equation} \label{discrete-variable-bandwidth-laplacian}
       L_{P}:= L_{\varepsilon,\alpha,\beta}^{}(p_i,p_j) := \dfrac{ \hat{K}_{\varepsilon,\alpha}^{}(p_i,p_j)-\delta_{ij}}{\varepsilon \rho(p_i)^2} =  \dfrac{ \hat{K}_{\varepsilon,\alpha}^{}(p_i,p_j)-\delta_{ij}}{\varepsilon q(p_i)^{2\beta}}.
    \end{equation} where we have set $\rho = q^\beta$. 
\end{enumerate}

\begin{remark}
The above construction is an instantiation for variable-bandwidth kernels of a general  ``3-kernel procedure" within the diffusion maps literature:
    \begin{enumerate}
        \item An unnormalized variable bandwidth kernel $K_\varepsilon^{}$
        \item An $\alpha$-density weighted kernel $K_{\varepsilon,\alpha}^{}$. 
        \item The row-normalized (Markov) kernel $\hat{K}_{\varepsilon,\alpha}$.
    \end{enumerate} 
    Note that when $\alpha=0$ the second step is skipped. When $\beta = 0$ then the recovered operator equals the classical (fixed-bandwidth) diffusion maps Laplacian \cite{coifman2006diffusion}.
\end{remark}

There are a number of quantities from the idealized discussion above that need to estimated in order to calculate the discrete Laplacian in practice. These are:
\begin{enumerate}
    \item Empirical density estimate $q_0$ of the true sampling density $q$. 
    \item Number of \( k \)-nearest neighbors to use (rather than fully-connected).
    \item Bandwidth choice $\varepsilon$.
    \item Dimension parameter $d$, or estimate thereof. 
\end{enumerate}
We discuss these in turn.

\paragraph{Empirical density estimate} In practice, we do not have access to the ground truth sampling density $q$, so it needs to be estimated. We follow the estimator offered in~\cite{Berry2016}. For each $p_i \in P$, let $I(i,j)$ denote the index of the $j$th nearest neighbor of $p_i$.
Define the local scale
\[
\rho_0(p_i)
:=
\left(
\frac{1}{k_0-1}\sum_{j=2}^{k_0} \|p_i-p_{I(i,j)}\|^2
\right)^{1/2}.
\]
Set
\[
\varepsilon_0^{1/2}
:=
\frac{1}{N}\sum_{i=1}^N \rho_0(x_i),
\qquad
\widetilde{\rho}_0(p_i)
:=
\frac{\rho_0(x_i)}{\varepsilon_0^{1/2}}.
\]


Now let $\{p_i\}_{i=1}^n \subset \mathbb{R}^D$ be data sampled from a density $q$ on a
$d$-dimensional manifold $M \subset \mathbb{R}^n$. Using the symmetric kernel with local bandwidth $\rho_0$, define
\[
q_0(p_i)
:=
(2\pi)^{-d/2}\frac{1}{\rho_0(p_i)^d\,N}
\sum_{l=1}^N
\exp\!\left(
-\frac{\|x_i-x_l\|^2}{2\,\rho_0(p_i)\rho_0(p_l)}
\right).
\]
Equivalently,
\[
q_0(x_i)
=
(2\pi\varepsilon_0)^{-d/2}\frac{1}{\widetilde{\rho}_0(p_i)^d\,N}
\sum_{l=1}^N
\exp\!\left(
-\frac{\|p_i-p_l\|^2}
{2\varepsilon_0\,\widetilde{\rho}_0(p_i)\widetilde{\rho}_0(p_l)}
\right).
\] 
In \cite{Berry2016} it is shown that choosing \( \rho := q_{0}^\beta \) entails \( \rho=q^\beta + O(\varepsilon) \) with high probability. Our local theory (Sections~\ref{subsubsec:kolmogorov-convergence}-\ref{subsubsec:local-kfc-consistency}) is stated under the assumption \( \rho=q^\beta + O(\varepsilon) \), and thus applies to this empirically-determined bandwidth choice with high probability. 


\begin{remark}
The theory in \cite{Berry2016}, and in turn our own (Appendix~\ref{Appendix:Subsec:theory}), are written in terms of the fully connected graph Laplacian. Following \cite{Berry2016}, we use the \( k \)-NN graph as an approximation for efficiency in practice, but note that our forthcoming convergence results are not strictly applicable in this case.
\end{remark}

\paragraph{Notation} To reduce ambiguity, for the remainder of this appendix, we will notate discrete quantities on \( P \) with subscripts tracking the explicit hyperparameters they depend on (such as \( \varepsilon\)) in lieu of the notation \( (\_) _{P}\) adopted in the main text. Similarly, we will variously denote quantities associated to the Riemannian manifold \( (M,g) \) with subscripts \( (\_)_{M} \) or \( (\_)_{g} \), whichever is less ambiguous given the context. Finally, we will employ the `formal pullback' notation \( \langle P^{*} \omega, P^{*} \eta \rangle  \)  to denote the local inner product on discrete \( k \)-forms in order to emphasize that \( \omega, \eta \) live on the ambient space \( \mathbb{R}^{D} \).

\paragraph{Bandwidth estimation}
All of our experiments fix \( \varepsilon=1 \) without further tuning.  


\paragraph{Dimension estimation}
 The empirical density estimation with theoretical guarantees as described above requires explicit knowledge of the intrinsic dimension \(d\) of the manifold \( M \), and therefore our theory does too. In practice, we estimate intrinsic dimension via local PCA a la~\cite{fukunaga1971algorithm} when it is unknown. 



\subsubsection{Discrete Inner Products}
\label{subsubsec:discrete-global-inner-product}

Recall (cf.~\eqref{def:gramfield-global-inner-product}) that in the smooth setting, the global (\( L^2 \)) inner product between restricted \( k \)-forms is given by $$\langle \! \langle \iota ^*\omega,  \iota^* \eta \rangle \! \rangle_{g}=\int_{x \in M} \langle \iota^* \omega,  \iota^* \eta \rangle_{g}(x) \, dV(x) , $$
i.e. the average over $M$ with respect to the volume form. If the sampling density $q$ is taken into account, there is a second inner product $$\langle \! \langle \iota ^* \omega , \iota^*\eta \rangle \! \rangle_{g, q}= \int_{x \in M} \langle \iota^*\omega, \iota^*\eta \rangle_{g}(x) q(x)\, dV(x),  $$
which weighs this average by assigning higher weights to dense regions. Which inner product to emulate in the discrete case depends on whether one would like to classify manifolds or manifolds \textit{with densities}. These correspond to averaging the discrete local inner product $\langle \_, \_ \rangle_{\varepsilon}$ against one of the following discrete measures on $P$:
\begin{enumerate}
    \item (Density-corrected) 
 \begin{equation}
        \langle \! \langle P^*\omega,P^* \eta \rangle \! \rangle_{\varepsilon} := \dfrac{1}{N}\sum_{p_i} \langle \omega ({p_i}), \eta({p_i})\rangle_{\varepsilon} \dfrac{1}{q(p_i)} { \approx \dfrac{1}{N}\sum_{p_i} \langle \omega_{p_i}, \eta_{p_i}\rangle_P \dfrac{1}{q_0(p_i)}}
    \end{equation}
    \item (Non-density corrected)
    \begin{equation}
        \langle \! \langle \omega, \eta \rangle \! \rangle_{\varepsilon,q}:= \dfrac{1}{N}\sum_{p_i} \langle \omega({p_i}), \eta({p_i})\rangle_{\varepsilon}
    \end{equation}
\end{enumerate}

Corollary~\ref{cor:global-continuum-limit} will formalize the validity of this intuition.

\subsection{Consistency Theory for \( k \)FC Discretization}
\label{Appendix:Subsec:theory}

\label{appendix:B2}



    

The goal of this subsection is to prove a (formal version of) Theorem~\ref{ithm:main-text-continuum-limits} from the main text and examine its consequences. Our central result, from which Theorem~\ref{ithm:main-text-continuum-limits} is derived, is as follows.

\newtheorem*{theorem*}{Theorem~\ref{thm:locla-consistency-estimate} (Informal)} \begin{theorem*}
    Let $p_{i} \in M$, $\omega,\eta \in \Omega^{k}(\mathbb{R}^{D})$. With high probability, $$\langle P^{*}\omega, P^{* }\eta \rangle_{\varepsilon, n}(p_{i})= \langle \iota^{*}\omega, \iota^{*}\eta \rangle_{g}(p_{i}) + O_{k, \omega ,\eta,p_{i}} \left(  \varepsilon, \frac{q(p_{i})^{(1-d \beta)/2}}{\sqrt{ n } \varepsilon^{2+d/4}}, (1+2\max_{ \ell }  |x^{\ell}(p_{i})|) \frac{q(p_{i})^{-c_{2}}}{\sqrt{ n } \varepsilon^{1/2+d/4}}  \right)  $$
for a constant \( c_{2} \) depending on \( \alpha, \beta \), and this estimate can be made uniform in $p_{i}$ provided $M$ is compact. 

\end{theorem*}

Our theory diverges from implementation in two main ways. First, it assumes fully connected kernel graphs, whereas the implementation uses \(k\)-NN truncation for efficiency. Certifying compatibility with \( k \)-NN truncation is nontrivial and left for future work.
Second, it assumes knowledge of the intrinsic dimension \(d=\dim M\). Beyond these two assumptions, a feature of our theory is its faithfulness to practice. Our results treat nonuniform sampling from an unknown density, variable bandwidths, and explicit finite-sample dependence upon (and provable insensitivity to) normalization parameters. 



Throughout this section, we assume \( \rho=q^\beta + O(\varepsilon) \). This holds with high probability e.g. for the empirical \( \rho \) used in our experiments (cf. Section~\ref{subsubsec:cdc-estimation-from-data-including-variable-bandwidth-laplacian-contruction}) \cite{Berry2016}. In particular, our central consistency result (Theorem~\ref{thm:locla-consistency-estimate}) does not require the true sampling density \( q \) to be known.

\subsubsection{Kolmogorov Convergence}

\label{subsubsec:kolmogorov-convergence}
The theory of variable-bandwidth diffusion kernels is supported by the following result, due to Berry and Harlim~\cite{Berry2016}. 
\begin{proposition}[\cite{Berry2016}, Corollary \( 1 \)]
\label{prop:BH-cor1}
Assume \( \rho=q^\beta  + O(\varepsilon)\). Let \( p_{i} \in M \), \( f \in L^2 \cap C^3(M) \). Define $$\delta_{\varepsilon, n, \alpha, \beta, f}(p):=\max \left(\varepsilon, \frac{q(p_{i})^{(1- d \beta) / 2}}{\sqrt{ n }\varepsilon^{2+d/4}}, \frac{\|\nabla f(p_{i})\| q(p_{i})^{-c_{2}}}{\sqrt{ n } \varepsilon^{1/2+d/4}}\right).$$
With high probability,
$$L_{\varepsilon, \alpha, \beta}f(p_{i})= \mathcal{L} _{\alpha, \beta}+ O\big(\delta_{\varepsilon, n, \alpha, \beta ,f}\big)(p_{i}) ,$$
where $$\mathcal{L}_{\alpha, \beta}= \frac{1}{q^{c_{1}}} \operatorname{div}(q^{c_{1}} \nabla f) = \Delta f +  c_{1} \nabla f \cdot \frac{\nabla q}{q}$$
for $c_{1}=2-2\alpha + d \beta + 2 \beta$ and $c_{2}=1 / 2 - 2\alpha + 2d \alpha + d \beta / 2 + \beta$. 

\end{proposition}

The Laplacian \( \Delta \) on \( M \) is recovered provided \( c_{1}=0 \), for instance when \( \alpha= \frac{1}{2} - \frac{d}{4} \), \( \beta=- \frac{1}{2} \). 

\begin{remark}
\label{remark:BH-purpose}
    If \( q \) is not bounded below (i.e. sampling is permitted to be arbitrarily sparse), then \( \delta_{\varepsilon, n, \alpha, \beta, f}  \) can explode as \( q \to 0 \) provided \( c_{2} > 0 \). The fixed-bandwidth case \( \beta=0 \) \textit{forces} \( c_{2} >0\); this concretely motivates variable-bandwidth kernels. Note that this scenario cannot occur if \( M \) is compact, since then \( q \) must be bounded below.\footnote{We will exploit this simple principle at length in Section~\ref{Appendix:Subsec:theory}.} While convergence results akin to Proposition~\ref{prop:BH-cor1} have been established for compact \( M \) since~\cite{coifman2006diffusion}, and bare convergence results in the variable-bandwidth case since~\cite{ting2010analysis}, the focus of \cite{Berry2016} is the explicit characterization of the bound in terms of the parameters \( \alpha, \beta \) and the true sampling density \(  q\) in the potentially noncompact case. Other heuristic justifications for variable-bandwidth diffusion kernels include:
    \begin{enumerate}
        \item They are less sensitive to \( \varepsilon \) than their fixed-bandwidth counterparts;
        \item They more effectively and transparently account for density.
    \end{enumerate}
\end{remark}

\begin{remark}
\label{rmk:whp-explanation}
   {By "with high probability", we mean the conclusion of Proposition~\ref{prop:BH-cor1} holds on an event $E_{\varepsilon, n, i , f}=A_{\varepsilon, n} \cap B_{\varepsilon, n, i ,f}$ satisfying}
\begin{equation}
\mathbb{P}(E_{\varepsilon, n, i, f}^{c}) \leq C_A n  \, e^{- c_{A}n \varepsilon^{4+ d/ 2}}+C_{B}e^{-c_{B,i, f}n \varepsilon^{3+d/2}}
\end{equation}
{for some constants $C_{A},C_{B},c_{A},c_{B,i, f}>0$ corresponding to the latter two terms of \( \delta_{\varepsilon, n, \alpha, \beta, f} \).}\footnote{Specifically, the event $A_{\varepsilon, n}$ corresponds to the second term $\frac{q(p_{i})^{(1- d \beta)/2}}{\sqrt{ n }\varepsilon^{2+ d/ 4}}$ of $\delta_{\varepsilon, n, \alpha, \beta, f}(p_{i})$ uniformly in $i$ and independent of $f$, while $B_{\varepsilon, n, i, f}$ corresponds to the third term $\frac{\|\nabla f(p_{i})\| q(p_{i})^{-c_{2}}}{\sqrt{ n } \varepsilon^{1/2+d/4}}$. (The first term $\varepsilon$ of $\delta_{\varepsilon, n, \alpha, \beta, f}(p)$ is deterministic, cf. Appendix A of~\cite{Berry2016}.); because they can be black-boxed without loss of continuity, we further defer a detailed derivation of these bounds to Appendix~\ref{subsubsec:the-bh-good-event}.}


\end{remark}

We are ultimately interested in estimating maps built from quantities of the form \( L_{\varepsilon, \alpha, \beta} f\) as \( f \) \textit{varies} over some finite function family \( F \). Toward this end, we collect the following corollary of Proposition~\ref{prop:BH-cor1}.

\begin{lemma}[Uniformizing in \( f \)] Let \( p_{i} \in M. \). Given ${F} \subset C^{3}(M) \cap L^{2}(M, q) $ finite, define $$\delta_{\varepsilon, n, \alpha, \beta, {F}}(p_{i}):=\max_{f \in {F}} \{ \delta_{\varepsilon, n, \alpha, \beta, f}(p_{i})\}.$$
 With high probability, 
    \begin{equation}
        L_{\varepsilon, \alpha, \beta}f(p_{i})= \mathcal{L} _{\alpha, \beta}f(p_{i})+ O\big(\delta_{\varepsilon, n, \alpha, \beta ,F}(p_{i})\big) \text{ for all }f \in F.
        \label{eqn:unif-f}
    \end{equation}
    Specifically, the above holds on an event \( E_{\varepsilon, n, i} \) with failure probability \( \mathbb{P}(E_{\varepsilon, n, i}^{c})  \leq C_{A} n e^{-c_{A}n \varepsilon^{4+d/2}}+C_{B} \# F e^{-c_{B, i, F}n \varepsilon^{3+d/2}}\).
    \label{lemma:uniformizing-in-f}
\end{lemma}

\begin{proof}
    Importing the notation and context of Remark~\ref{rmk:whp-explanation}, we see that the claim holds on the event $E_{\varepsilon, n, i}:=A_{\varepsilon, n} \cap \bigcap_{f \in F}^{}B_{\varepsilon, n, i , f}$, whose failure probability may be estimated $$\mathbb{P}(E^{c}_{\varepsilon, n, i}) \leq \mathbb{P}(A_{\varepsilon, n}^{c})+ \sum_{f \in F}\mathbb{P}(B^{c}_{\varepsilon, n,i,f})\leq C_{A} n e^{-c_{A}n \varepsilon^{4+d/2}}+C_{B} \# F e^{-c_{B, i}n \varepsilon^{3+d/2}},$$
where $c_{B,i}=\min_{f \in F} c_{B,i,f}$. 
\end{proof}

At times, we will also be interested in statements holding uniformly on \( P \). A sufficient condition is compactness of \( M \).\footnote{We reiterate that compactness of \( M \) is not a standing assumption of the present section; indeed, Lemma~\ref{lemma:uniformizing-in-i-with-compactness} will not see heavy use until the global convergence proofs of Section~\ref{subsubsec:global-kfc-consistency}. }
\begin{lemma}
Assume in addition that $M$ is compact. Let $F \subset C^{3}(M) \cap L^{2}(M,q)$ be finite, and define
\begin{equation}
\overline{\delta}_{\varepsilon, n, \alpha, \beta, F}:=\sup_{p \in M}\delta_{\varepsilon,n,\alpha, \beta, F}(p).
\end{equation}
Then $\overline{\delta}_{\varepsilon, n, \alpha, \beta ,F}<\infty$, and with high probability,
\begin{equation}
L_{\varepsilon, \alpha, \beta}f(p_{i}) = L_{\alpha, \beta}f(p_{i})+O(\overline{\delta}_{\varepsilon, n, \alpha, \beta, F}) \text{ for all }i \in [n] \text{ and }f \in F.
\end{equation}
Specifically, the above holds on an event $E_{\varepsilon, n}$ with failure probability
\begin{equation}
\mathbb{P}(E_{\varepsilon, n}^{c}) \leq C_{A} n e^{-c_{A}n \varepsilon^{4+d/2} }+C_{B} n \# F e^{-c_{B}n\varepsilon^{3+d/2}}.
\end{equation}
\label{lemma:uniformizing-in-i-with-compactness}
\end{lemma}

\begin{proof}

Import the notation and context of Remark~\ref{rmk:whp-explanation} and Lemma~\ref{lemma:uniformizing-in-f}. 

Since $M$ is compact and $q>0$ is continuous, there exist constants $q_{\text{min}}$, $q_{\text{max}}$ such that $0<q_{\text{min}} \leq q \leq q_{\text{max}}$. Since $F \subset C^{3}(M)$, each $\|\nabla f\|_{\infty}<\infty$, and then since $F$ is finite, $Z_{F}:=\max_{f \in F}\|\nabla f\|_{\infty}<\infty$.  Therefore $\delta_{\varepsilon, n, \alpha, \beta, F}:M \to \mathbb{R}$ is bounded on $M$, hence $\overline{\delta}_{\varepsilon, n, \alpha, \beta, F}<\infty$, and moreover the pointwise constant $c_{B,i}=\min_{f \in F} c_{B, i, f}$, $c_{B,i,f}=c_{B,i,f}(\hat{a}_{2}, c_{2}, c, \|\nabla f(p_{i})\|, q(p_{i}))$, is uniformly bounded in $i$, say, by some $0<c_{B}\leq \inf_{i \in [n]}c_{B,i}$. 

We see that the claim holds on the event $E_{\varepsilon, n}:= \bigcap_{i=1}^{n} E_{\varepsilon , n, i}=A_{\varepsilon, n} \cap \bigcap_{i=1}^{n}B_{\varepsilon, n, i}$, whose failure probability may be estimated
\begin{align}
\mathbb{P}(E_{\varepsilon, n}^{c})&\leq \mathbb{P}(A_{\varepsilon, n}^{c})+ \sum_{i=1}^{n} \mathbb{P}(B_{\varepsilon, n, i}^{c}) \\
&\leq C_{A}n e^{-c_{A}n \varepsilon^{4+d/2}}+C_{B}n\# F e^{-c_{B}n \varepsilon^{3+d/2}}
\end{align}
where we have used Lemma~\ref{lemma:uniformizing-in-f}. Since $\delta_{\varepsilon, n ,\alpha, \beta, F}(p_{i})\leq  \overline{\delta}_{\varepsilon, n , \alpha, \beta, F},$ the claimed uniform error follows. 

\end{proof}

\subsubsection{Carré du Champ Consistency}
\label{subsubsec:cdc-consistency}
An emerging paradigm in geometric data analysis leverages the convergence of the Laplacian \( \Delta \) to approximate the Riemannian metric $g$ of a `submanifold' specified through point cloud data \cite{giannakis-berry2020,jones2024diffusiongeometry}. Access is provided via the carré du champ identity
\begin{equation}
    \Gamma_{}(f,h)(p) := \Big( f \Delta_M h + h \Delta_M f - \Delta_M(fh)\Big)(p) = \langle df,dh\rangle_g(p)
    \label{eqn:one-of-many-cdc-identities}
\end{equation} characterizing the metric \( g \) in terms of \( \Delta \) (See Appendix~\ref{app:background-diff-forms}). The work~\cite{jones2026manifolddiffusiongeometrycurvature} applies Proposition~\ref{prop:BH-cor1} to consistently estimate \( \Gamma \) on compact \( M \). We expand the brief argument provided there to examine the roles of \( \alpha \), \( \beta \), and noncompactness. Put \[
2\Gamma_{\varepsilon, \alpha, \beta}(f,h):=\hat{f} L_{\varepsilon, \alpha, \beta}\hat{h} + \hat{h}L_{\varepsilon, \alpha, \beta}\hat{f}-L_{\varepsilon, \alpha, \beta}(\hat{f}\hat{h}),
\]
where $\hat{f}:=P^{*}f$, $\hat{h}=P^{*}h$. 
\begin{corollary}[Based on~\cite{jones2026manifolddiffusiongeometrycurvature}, Corollary 3.2]
\label{cor:cdc-consistency}
Let \( p_{} \in M \), and assume \( F:= \{ f,h, fh \}  \subset C^{3}(M) \cap L^{2}(M,q) \). With high probability, 

\[
\Gamma_{\varepsilon, \alpha, \beta} (f, h)(p_{i}) = \langle df, dh \rangle_{g} (p_{i}) + O \big( ( 1 + |f(p_{i})| + |h(p_{i})| ) \delta_{\varepsilon, n, \alpha, \beta, F}(p_{i}) \big).
\]

Specifically, the above holds on an event \( E_{\varepsilon, n, i} \) with failure probability \( \mathbb{P}(E_{\varepsilon, n, i}^{c})  \leq C_{A} n e^{-c_{A}n \varepsilon^{4+d/2}}+C_{B} \# F  e^{-c_{B, i, F}n \varepsilon^{3+d/2}}\).

\begin{remark}
    If in fact \( f,h \in L^\infty (M) \) (such as when \( M \) is compact), then automatically \( f h \in C^3(M) \cap L^2(M) \) and Corollary~\ref{cor:cdc-consistency} applies to produce 
\[
\Gamma_{\varepsilon,\alpha,\beta}(f,h)(p_i)
=
\langle df,dh\rangle_g(p_i)
+
O\!\left(
\varepsilon,\,
\frac{q(p_i)^{(1-d\beta)/2}}{\sqrt n\,\varepsilon^{2+d/4}},\,
\frac{\bigl(\|f\|_\infty+\|h\|_\infty+\|\nabla f(p_i)\|+\|\nabla h(p_i)\|\bigr)\,q(p_i)^{-c_2}}
{\sqrt n\,\varepsilon^{1/2+d/4}}
\right)
\]
with (the same) high probability. Upon setting \( \alpha=\frac{1}{2}-\frac{d}{4} \), \( \beta=-\frac{1}{2} \), this statement specializes to recover Corollary 3.2 in~\cite{jones2026manifolddiffusiongeometrycurvature}. 

\end{remark}

\end{corollary}



Applying Lemma~\ref{lemma:uniformizing-in-f} to $F$ produces an event $E_{\varepsilon, n ,i}$ with the stated failure probability on which $$L_{\varepsilon, \alpha, \beta}f'(p_{i}) = \mathcal{L}_{\alpha, \beta}f'(p_{i}) +O(\overbrace{ \delta_{\varepsilon, n, \alpha, \beta, F}(p_{i}) }^{ =: \delta_{F}(p_{i}) }) \text{ for all } f' \in F.$$
We have \begin{align}
2 \Gamma_{\varepsilon, \alpha, \beta}(f,h)(p_{i}) &= f(p_{i}) [\mathcal{L}_{\alpha, \beta}h(p_{i}) + O\big(\delta_{F}(p_{i})\big)] + h(p_{i})[\mathcal{L}_{\alpha, \beta}f(p_{i})+O\big( \delta_{F}(p_{i}) \big)] - \mathcal{L}_{\alpha, \beta}(fh)(p_{i})+O\big( \delta_{F}(p_{i}) \big) \\
&= 2 \Gamma_{\mathcal{L}_{\alpha, \beta}} (f, h)(p_{i}) + O \big ( (1 + |f(p_{i})|+|h(p_{i})|) \delta_{F}(p_{i}) \big).
\end{align}
where $\Gamma_{\mathcal{L_{\alpha, \beta}}}$ is the carré du champ associated to the continuum Kolmogorov operator $\mathcal{L_{\alpha, \beta}}$. In fact, we claim $\Gamma_{\mathcal{L}_{\alpha, \beta}}(f, h)=\underbrace{ \langle df, dh \rangle_{g} }_{ \Gamma(f,h) }$, the carré du champ~\eqref{eqn:one-of-many-cdc-identities} on $(M,g)$. Indeed, $\mathcal{L}_{\alpha, \beta}$ consists of a diffusion term $\Delta$ and a drift term $c_{1} \frac{\nabla q}{q} \cdot \nabla$. The carré du champ constructed via \( \Delta \) recovers \( \langle df, dh \rangle  _{g}\). The carré du champ constructed via $c_{1} \frac{\nabla q}{q} \cdot \nabla$, 
 \begin{equation} 
    \Gamma_{ c_1 \nabla q/q \cdot \nabla} (f,h)(p) =c_1 \dfrac{\nabla q}{q} \Big(  f \nabla h + h \nabla f - \nabla(fh) \Big)(p) ,
    \end{equation}
is zero by the Leibniz rule for \( \nabla \). 

\begin{remark}[Parameter choices]
We see the parameters \( \alpha, \beta \) do not affect the limiting form, only rates (through their interaction with \( q \)). It is suggested in~\cite{Berry2016} to take \( \beta=-\frac{1}{2} \). From there, \( \alpha \) only has to be chosen to produce \( c_{2}<0 \) (so that \( \delta_{\varepsilon, n ,\alpha, \beta, F} \) does not explode in sparsely sampled regions where \( q \) is small, cf. Remark~\ref{remark:BH-purpose}). A sufficient choice is just \( \alpha=0 \). This suggests that \textit{\( \alpha \)-normalization is not necessary to do diffusion geometry}, although rigorous experimentation to confirm or deny such a hypothesis is tangential to the scope of this work and hence omitted. 
    \label{rmk:parameter-choices}
\end{remark}

\subsubsection{Consistency of the Gram Fields \( G^{(k)} \)}
\label{subsubsec:tangential-projector-consistency}

For the remainder of this appendix subsection (Appendix~\ref{appendix:B2}), we suppress \( \alpha \) and \( \beta \) from the notation to emphasize that our results depend on them only by way of finite-sample bounds rather than limiting operators (cf. Remark~\ref{rmk:parameter-choices}). Officially, the norm \( \|A\| \) of a matrix \( A \) in this section is the max-norm, although our estimates would hold for any other norm as well due to finite-dimensionality. 

\begin{proposition}[tangential projector consistency]

Let $p \in M$. Assume that for all $i,j \in[D]$, $\{ x^{i}|_{M}, x^{j}|_{M}, x^{i}x^{j} |_{M} \} \subset C^{3}(M) \cap L^{2}(M,q)$. Set $F:=\{ x^{i} |_{M} \}_{i=1}^{D} \cup \{ x^{i}x^{j} |_{M} \}_{1 \leq i \leq j \leq D} \subset C^{3}(M) \cap L^{2}(M)$. Define \[
\mathfrak{e}_{\varepsilon, n}(p):=\max \{  \varepsilon, \frac{q(p)^{(1-d \beta)/2}}{\sqrt{ n } \varepsilon^{2+d/4}}, (1+2\max_{1 \leq \ell \leq D}  |x^{\ell}(p)|) \frac{q(p)^{-c_{2}}}{\sqrt{ n } \varepsilon^{1/2+d/4}}  \}
\]. With high probability,  

\[
\| G_{M}(p) - G_{\varepsilon, n}(p)  \|_{} = O( \mathfrak{e}_{\varepsilon, n}(p) ),
\]
that is, 

$$\big(G_{\varepsilon, n}(p)\big)^{ij}=G ^{ij}(p)+O( \mathfrak{e}_{\varepsilon, n}(p) ) \text{ for all }i,j \in [D] .$$

Specifically, the above holds on an event $E_{\varepsilon, n,p}$ with failure probability $$\mathbb{P}(E_{\varepsilon, n, p}^{c}) \leq C_{A} n e ^{-c_{A}n \varepsilon^{4+d/2}}+ \frac{D(D+3)}{2} C_{B}e^{-c_{B, i, F}n \varepsilon^{3+d/2}}$$
\label{prop:tangential-projector-consistency}
\end{proposition}
The bound above works uniformly over all entries of $G_{\varepsilon, n}^{(k)}(p)$ up to a $k$-dependent constant:

\begin{corollary}[k$th$ tangential projector consistency]
Fix assumptions and notation as in Proposition~\ref{prop:tangential-projector-consistency}, and let \( k \geq 1 \). With high probability,
$$\|G^{(k)}_{M}(p)-G^{(k)}_{\varepsilon, n}(p)\| \leq R^{(k)}_{\varepsilon, n}(p),$$
where $|R^{(k)}_{\varepsilon, n}(p)|\lesssim k!\big( (1+{e}_{\varepsilon, n}(p))^{k}-1\big)^{}$ and $e_{\varepsilon,n}(p)$ is $O(\mathfrak{e}_{\varepsilon,n}(p))$. It follows that $$\|G_{M}^{(k)}(p) - G_{\varepsilon, n}^{(k)}(p)\|=O_{k}(\mathfrak{e}_{\varepsilon, n}(p)).$$

Specifically, the above holds on the same event \( E_{\varepsilon, n ,p} \) as Proposition~\ref{prop:tangential-projector-consistency}.

\label{cor:kth-tangential-projector-consistency}
\end{corollary}

\begin{remark}
    Assuming the data is mean-centered, this implies an intuitive (although less tight) bound in terms of the diameter of $M$. 
\end{remark}

\begin{proof}[Proof of Proposition~\ref{prop:tangential-projector-consistency}]
Set $F:=\{ x^{i} |_{M} \}_{i=1}^{D} \cup \{ x^{i}x^{j} |_{M} \}_{1 \leq i \leq j \leq D} \subset C^{3}(M) \cap L^{2}(M)$. 
Applying Lemma~\ref{lemma:uniformizing-in-f} to the finite family \( F \) produces an event
\(E_{\varepsilon,n,p}\) with \[\mathbb P(E_{\varepsilon,n,p}^c)\leq C_A n e^{-c_A n\varepsilon^{4+d/2}}+\frac{D(D+3)}2\,C_B e^{-c_{B,p,F_{}}n\varepsilon^{3+d/2}}
\]
on which, for all \( i, j \in [D] \), the operators
\( L_{\varepsilon,\alpha,\beta}(x^i|_M)(p), L_{\varepsilon,\alpha,\beta}(x^j|_M)(p), L_{\varepsilon,\alpha,\beta}(x^i x^j|_M)(p) \) simultaneously satisfy the estimate~\eqref{eqn:unif-f}. Applying the estimates from Lemma~\ref{lemma:uniformizing-in-f} yields
\[(G_{\varepsilon, n}(p))^{ij}=\langle dx^i,dx^j\rangle_g(p)+O\left(\varepsilon, \frac{q(p)^{(1-d\beta)/2}}{\sqrt n\varepsilon^{2+d/4}}, \frac{\Xi(p)\,q(p)^{-c_2}}{\sqrt n\varepsilon^{1/2+d/4}}\right),
\] where we have used pullback naturality to obtain $\Gamma_{M}(\iota^{*}x^{i}, \iota^{*}x^{j})=\langle \iota^{*}dx^{i}, \iota^{*}dx^{j} \rangle$ and 
\[
\Xi(p):=\max\Bigl\{\|\overbrace{ \nabla(x^i|_M)(p)\| }^{ \leq 1 },\overbrace{ \|\nabla(x^j|_M)(p)\| }^{ \leq 1 },\overbrace{ \|\nabla(x^i x^j|_M)(p)\| }^{=\overbrace{ \| x^{i} \nabla(x^{j} |_{M}) +x^{j} \nabla(x^{i} |_{M}) \| }^{ \leq 1 + 2 \max_{\ell \in [D]}|x^{\ell}(p)| } }\Bigr\}\leq 1+2 \max_{\ell \in [D]}|x^{\ell}(p)|. 
\]
\end{proof}

\begin{proof}[Proof of Corollary~\ref{cor:kth-tangential-projector-consistency}]
The proof is combinatorial. Assume the event \( E_{\varepsilon, n, p} \) from the proof of Proposition~\ref{prop:tangential-projector-consistency} holds. Obtain $C>0$ such that, on $E_{\varepsilon, n, p}$, 
$$|\big( G_{\varepsilon, n}(p) \big)^{ij} - \big(G_{M}(p)\big)^{ij}| \leq e_{\varepsilon, n}(p):= C \max \left\{ \varepsilon, \frac{q(p)^{(1-d \beta)/2}}{\sqrt{ n } \varepsilon^{2+d/4}}, (1+2\max_{1 \leq \ell \leq D}  |x^{\ell}(p)|) \frac{q(p)^{-c_{2}}}{\sqrt{ n } \varepsilon^{1/2+d/4}}  \right\}\text{ for all }i,j \in [D].$$
In the following, we suppress $p$ from the notation.
Let $I,J \in \operatorname{Asc}(k,D)$. Writing $G_{\varepsilon, n}^{i_{s}j_{r}}=G^{i_{s}j_{r}}+\kappa^{i_{s}j_{r}}$ for some $|\kappa^{i_{s}j_{r}}|\leq e_{\varepsilon, n}$, for all $i_{s} \in I$, $j_{s} \in J$, one has
\begin{align} \det  {{G}^{}}^{IJ}_{\varepsilon, n}&= 
\det_{s,r \in [k]} [{G}_{\varepsilon, n}^{i_{s}, j_{r}}] \\&= \det_{s,r \in [k]}[ G^{i_{s}j_{r}}+ \kappa^{i_{s}j_{r}}] \\
&= \sum_{\sigma \in S_{k}} \operatorname{sgn}(\sigma) \prod_{s=1}^{k} ( G^{i_{s} j_{\sigma(s)}} + \kappa^{s, \sigma(s)}) \\
&= \sum_{\sigma \in S_{k}} \operatorname{sgn}(\sigma) \sum_{L \in  2^{[k]}} \left( \overbrace{  \prod_{s \in L}^{}  G^{i_{s} j_{\sigma(s)}}  }^{ |\cdot| \leq 1 } \right) \left(\overbrace{  \prod_{s \not \in L}^{} \kappa^{s, \sigma(s)} }^{ |\cdot| \leq e_{\varepsilon, n} ^{\#L^{c}}} \right)  \\
&= \sum_{\sigma \in S_{k}} \operatorname{sgn}(\sigma) \prod_{s=1}^{k} G^{i_{s }j_{\sigma(s)}  }+ \sum_{\sigma \in S_{k}} \operatorname{sgn}(\sigma) \sum_{[k] \neq L \in  2^{[k]}}^{} \left( \overbrace{  \prod_{s \in L}^{}  G^{i_{s} (j_{\sigma(s)})}  }^{ |\cdot| \leq 1 } \right) \left(\overbrace{  \prod_{s \not \in L}^{} \kappa^{s, \sigma(s)} }^{ |\cdot| \leq e_{\varepsilon, n} ^{\#L^{c}}} \right) \\
&= \det G^{IJ} + R_{\varepsilon, n} ^{IJ}\text{ on }E_{\varepsilon, n, p}, \text{ where } |R_{\varepsilon, n}^{IJ}| \leq k! \sum_{a=1}^{k} {k \choose a} e_{\varepsilon, n}^{a}=k ! \big( ( 1+ e_{\varepsilon, n} )^{k} -1 \big).
\end{align}


\end{proof}
If \( M \) is compact, then Corollary~\ref{cor:kth-tangential-projector-consistency} and Lemma~\ref{lemma:uniformizing-in-i-with-compactness} immediately combine to obtain the following.

\begin{corollary}
\label{cor:kth-tangential-projector-compactness-corollary}
    Fix and assumptions and notations as in Corollary~\ref{cor:kth-tangential-projector-consistency}. Further assume \( M \) is \emph{compact}. Define \[\overline{\mathfrak{e}}_{\varepsilon,n}:=\sup_{p\in M} \mathfrak{e}_{\varepsilon,n}(p).\] Then \( \overline{\mathfrak{e}}_{\varepsilon, n} <\infty \), and with high probability,
\[ \|G_{M}^{(k)}(p_{i}) - G_{\varepsilon, n}^{(k)}(p_{i})\|=O_{ k}(\overline{\mathfrak{e}}_{\varepsilon, n})   \text{ for all } i\in[n].\]
Specifically, the above holds on an event \(E_{\varepsilon,n}\) with failure probability
\[
\mathbb{P}(E_{\varepsilon,n}^{c})
\le
C_{A} n e^{-c_{A}n \varepsilon^{4+d/2}}
+
\frac{D(D+3)}{2} C_{B} n e^{-c_{B}n \varepsilon^{3+d/2}}.
\]

\end{corollary}

\subsubsection{Local \( k \)FC Consistency}
\label{subsubsec:local-kfc-consistency}

We arrive at the central result of this section.

\begin{theorem}[Local Consistency Estimate]
\label{thm:locla-consistency-estimate}
Let $\omega, \eta \in \Omega^{k}(\mathbb{R}^{D})$.
\begin{enumerate}
    \item[(1 - Pointwise)] Let $p \in M$. With high probability, \begin{equation}
        \langle P^{*}\omega, P^{*}\eta \rangle_{\varepsilon, n}(p) = \langle \iota^{*}\omega, \iota^{*}\eta \rangle_{g}(p)  +R_{\varepsilon, n, \omega, \eta}(p),
\label{eqn:local-pointwise-consistency}
    \end{equation} 
where $|R^{}_{\varepsilon, n, \omega, \eta}(p)|\leq C_{\omega, \eta, p} k!\big( (1+{e}_{\varepsilon, n}(p))^{k}-1\big)^{}$and $e_{\varepsilon,n}(p)$ is $O(\mathfrak{e}_{\varepsilon,n}(p))$. It follows that $$ | (\langle \iota^{*}\omega, \iota^{*}\eta \rangle_{g} - \langle P^{*}\omega, P^{*}\eta \rangle_{\varepsilon, n})(p) | =O_{k, \omega, \eta }(\mathfrak{e}_{\varepsilon, n}(p)).$$

Specifically, the above holds on an event $E_{\varepsilon, n,p}$ with failure probability $$\mathbb{P}(E_{\varepsilon, n, p}^{c}) \leq C_{A} n e ^{-c_{A}n \varepsilon^{4+d/2}}+ \frac{D(D+3)}{2} C_{B}e^{-c_{B, p}n \varepsilon^{3+d/2}}.$$
\item[(2 - Uniform)] If \(M\) is \emph{compact}, define \[\overline{\mathfrak{e}}_{\varepsilon,n}:=\sup_{p\in M} \mathfrak{e}_{\varepsilon,n}(p).\] Then \( \overline{\mathfrak{e}}_{\varepsilon, n} <\infty \), and with high probability,
\[\big|\langle P^{*}\omega, P^{*}\eta \rangle_{\varepsilon,n}(p_i)-\langle \iota^{*}\omega, \iota^{*}\eta \rangle_{g}(p_i)\big|=O_{ k,\omega, \eta}(\overline{\mathfrak{e}}_{\varepsilon, n})   \text{ for all } i\in[n]. \label{eqn:local-uniform-consistency}\]
Specifically, the above holds on an event \(E_{\varepsilon,n}\) with failure probability
\[
\mathbb{P}(E_{\varepsilon,n}^{c})
\le
C_{A} n e^{-c_{A}n \varepsilon^{4+d/2}}
+
\frac{D(D+3)}{2} C_{B} n e^{-c_{B}n \varepsilon^{3+d/2}}.
\]
\end{enumerate}

\end{theorem}

\begin{proof}
\textbf{Pointwise estimate.} For notational ease, we employ the Einstein summation convention and suppress $p$ unless its inclusion is required to avoid ambiguity. Using that $\Omega^{k}(\mathbb{R}^{D})$ is a free $C^{\infty}(\mathbb{R}^{D})$-module, write $\omega=\omega_{I} \, dx^{I}$, $\eta=\eta_{J}\,dx^{J}$ for unique $\omega_{I} \in C^{\infty}(\mathbb{R}^{D})$. Invoking Corollary~\ref{cor:kth-tangential-projector-consistency}, we have on $E_{\varepsilon, n, p}$ that \begin{align}
\langle P^{*} \omega, P^{*}\eta \rangle_{\varepsilon, n} &= P^{*}\omega_{I}P^{*}\eta_{J} \langle P^{*} dx^{I}, P^{*}dx^{J} \rangle_{\varepsilon, n}  \\
&=  P^{*}\omega_{I}  P^{*}\eta_{J} (\det   {G} _{\varepsilon, n}^{IJ}) \\
&= P^{*}\omega_{I} P^{*}\eta_{J} (\det G^{IJ}+R_{\varepsilon, n}^{IJ}) \\
&= \underbrace{ P^{*}\omega_{I} P^{*}\eta_{J} \det G^{IJ}  }_{ =\langle \iota^{*} \omega, \iota^{*} \eta \rangle_{g}  }+ P^{*} \omega_{I} P^{*} \eta_{J} R_{\varepsilon, n}^{IJ}  \, (*).
\end{align}
Now \begin{align}
|P^{*}\omega_{I} P^{*}\eta_{J} R_{\varepsilon,n}^{IJ}| &\leq |P^{*}\omega_{I}||P^{*}\eta_{J}||R_{\varepsilon, n}^{IJ}| \\
& \leq \sum_{I, J } |\omega_{I}(p)| |\eta_{J}(p)|  k! \big( (1 + \mathfrak{e}_{\varepsilon,n}(p))^{k}-1 \big) \\
& \leq \underbrace{ \|\omega(p)\|_{\ell^{1}\big(\mathbb{R}^{{D \choose k}}\big)} \|\eta(p)\|_{\ell^{1}\big(\mathbb{R}^{{D \choose k}}\big)}  }_{ =: C_{\omega, \eta,p} }  k! \big( (1 + \mathfrak{e}_{\varepsilon,n}(p))^{k}-1 \big) ,\label{eqn:local-k-form-bound}
\end{align}
where in defining \( C_{\omega, \eta, p} >0\) we have identified \( \omega(p) \) with its coordinate vector in the \( (dx^I) \) basis to take the \( \ell^1 \)-norm.

\textbf{Uniform estimate.} Now assume \( M \) is compact. This will allow us to uniformize both the tangential projector error \( e_{\varepsilon, n}(p) \) and constant \( C_{\omega, \eta, p} \) in \( p \) with high probability, so that \eqref{eqn:local-k-form-bound} admits the further \( p \)-independent estimate \[
\eqref{eqn:local-k-form-bound} \leq C_{\omega, \eta} k! \big( (1 + \overline{\mathfrak{e}}_{\varepsilon, n} )^{k}-1  \big), \overline{\mathfrak{e}}_{\varepsilon, n}:= \sup_{p \in M}\mathfrak{e}_{\varepsilon, n}(p).
\]
Indeed, compactness and continuity of \( \omega_{I}, \eta_{J} \) allow us to immediately write \[
C_{\omega, \eta,p}=\|\omega_{}(p)\|_{\ell^{1}\big(\mathbb{R}^{{D \choose k}}\big)} \|\eta_{}(p)\|_{\ell^{1}\big(\mathbb{R}^{{D \choose k}}\big)}\leq \max_{p \in M} \|\omega(p)\|_{\ell^{1}\big(\mathbb{R}^{{D \choose k}}\big)}\, \max_{p \in M} \|\eta(p)\|_{\ell^{1}\big(\mathbb{R}^{{D \choose k}}\big)}=:C_{\omega, \eta}<\infty,
\]
and the required control of \( \mathfrak{e}_{\varepsilon, n} \) is supplied by Corollary~\ref{cor:kth-tangential-projector-compactness-corollary}.


\end{proof}

\begin{corollary}[Local Continuum Limit]
\label{cor:local-continuum-limit}
Upon taking \( \varepsilon_{n}:=n^{-\vartheta} \) for \( 0<\vartheta < \frac{2}{d+8} \), one has 
\begin{enumerate}
    \item[1. (Pointwise)] \refstepcounter{enumi}%
    For any \( p_{i}  \in P\), $$\lim_{n \to \infty} \langle P^{*}\omega, P^{*} \eta \rangle_{\varepsilon_{n}, n} (p_{i})=\langle \iota^{*}\omega, \iota^{*}\eta \rangle_{g}(p_{i}) \text{ a.s.}$$

    \item[2. (Uniform)] \refstepcounter{enumi}\label{item:local-continuum-limit-item-2}%
    If \( M \) is \emph{compact}, then \[ \sup_{i \in [n]}|\langle P^{*}\omega, P^{*}\eta \rangle_{\varepsilon_{n}, n}(p_{i})-\langle \iota^{*}\omega, \iota^{*}\eta \rangle_{g}(p_{i})| \xrightarrow{n\to \infty}0 \text{ a.s.} \] 
    
\end{enumerate}
    
\end{corollary}

\begin{proof}
\textit{Pointwise.} The idea is that by choosing an appropriate coupling of $\varepsilon$ and $n$ we get an explicit continuum limit out of Equation~\ref{eqn:local-pointwise-consistency}, in the sense that $\mathfrak{e}_{\varepsilon_{n}, n} \to 0$ as $n \to \infty$. Perhaps the most evident such coupling is $\varepsilon_{n}:=n^{-\vartheta}$ for $0<\vartheta< \frac{2}{d+8}$, which clearly sends each argument of $\mathfrak{e}_{\varepsilon_{n}, n}(p_i)$ to zero as $n\to \infty$. Define $r_{n}(i):=C_{\omega, \eta, p_{i}}k!\big( (1 + e_{\varepsilon_{n}, n}(p_{i}))^{k}-1 \big)$; note $r_{n}(i) \xrightarrow{n \to \infty}0$. By Theorem~\ref{thm:locla-consistency-estimate}, Part 1, there exists an event $E_{\varepsilon_{n},n,p_{i}}$ on which $$\underbrace{ |\langle P^{*}\omega, P^{*}\eta \rangle_{\varepsilon_{n},n}(p_{i}) -\langle \iota^{*}\omega, \iota^{*}\eta \rangle_{g}(p_{i}) | }_{ |R_{\varepsilon_{n}, n, \omega, \eta}(p_{i})| }\leq r_{n}(i)$$
with $\mathbb{P}(E_{\varepsilon_{n}, n, {i}}^{c})\leq C_{A}n e^{-c_{A}n \varepsilon_{n}^{4+d/2}}+C_{B}e^{-c_{B,i}n \varepsilon_{n}^{3+d/2}}$ up to independent constants. Substituting  $\varepsilon_{n}=n^{- \vartheta}$ gives $\mathbb{P}(E_{\varepsilon_{n}, n, i}^{c}) \leq ne^{-c_{A}n^{1-\vartheta(4+d/2)}}+e^{-c_{B, i}n^{1 - \vartheta(3+d/2)}}$ up to constants. Both exponents tend to $+\infty$ since $\vartheta< \frac{1}{4+\frac{d}{2}}=\frac{2}{d+8}$, implying $\sum_{n=1}^{\infty}\mathbb{P}(E_{\varepsilon_{n},n,i}^{c})<\infty$. Thus $\mathbb{P}(E_{\varepsilon_{n}, n, i}^{c} \text{ i.o.})=0$ by Borel-Cantelli, meaning that with probability $1$ there is $N$ large enough that $|R_{\omega, \eta, \varepsilon_{n}, n}(p_{i})| \leq r_{n}(i) \to 0$ for all $n \geq N$, i.e., $|R_{\omega, \eta, \varepsilon_{n}, n}(p_{i})| \to 0 \text{ a.s.}$

\textit{Uniform.} Assume $M$ is compact. The idea is again that by choosing an appropriate coupling of $\varepsilon$ and $n$ we get an explicit continuum limit out of Equation \ref{eqn:local-uniform-consistency}, now in the sense that $\overline{\mathfrak{e}}_{\varepsilon_{n}, n} \to 0$ as $n \to \infty$, where $\overline{\mathfrak{e}}_{\varepsilon,n}:=\sup_{p \in M}\mathfrak{e}_{\varepsilon,n}(p)$. The coupling $\varepsilon_{n}:=n^{-\vartheta}$, $0<\vartheta< \frac{2}{d+8}$, clearly sends each argument of $\overline{\mathfrak{e}}_{\varepsilon_{n}, n}$ to zero as $n\to \infty$. Define $\overline{r}_{n}:=C_{\omega, \eta}k!\big( (1 + \overline{e}_{\varepsilon_{n}, n})^{k}-1 \big)$; note $\overline{r}_{n} \xrightarrow{n \to \infty}0$. By Theorem~\ref{thm:locla-consistency-estimate}, Part 2, there exists an event $E_{\varepsilon_{n},n}$ on which
$$\sup_{i \in [n]}|\langle P^{*}\omega, P^{*}\eta \rangle_{\varepsilon_{n},n}(p_{i}) -\langle \iota^{*}\omega, \iota^{*}\eta \rangle_{g}(p_{i}) |
\leq \overline{r}_{n}$$
with $\mathbb{P}(E_{\varepsilon_{n}, n}^{c})\leq C_{A}n e^{-c_{A}n \varepsilon_{n}^{4+d/2}}+C_{B}n e^{-c_{B}n \varepsilon_{n}^{3+d/2}}$ up to independent constants. Substituting $\varepsilon_{n}=n^{- \vartheta}$ gives $\mathbb{P}(E_{\varepsilon_{n}, n}^{c}) \leq ne^{-c_{A}n^{1-\vartheta(4+d/2)}}+ne^{-c_{B}n^{1 - \vartheta(3+d/2)}}$ up to constants. Both exponents tend to $+\infty$ since $\vartheta< \frac{1}{4+\frac{d}{2}}=\frac{2}{d+8}$, implying $\sum_{n=1}^{\infty}\mathbb{P}(E_{\varepsilon_{n},n}^{c})<\infty$. Thus $\mathbb{P}(E_{\varepsilon_{n}, n}^{c} \text{ i.o.})=0$ by Borel-Cantelli, meaning that with probability $1$ there is $N$ large enough that
$$\sup_{i \in [n]}|\langle P^{*}\omega, P^{*}\eta \rangle_{\varepsilon_{n},n}(p_{i}) -\langle \iota^{*}\omega, \iota^{*}\eta \rangle_{g}(p_{i}) |
\leq \overline{r}_{n} \to 0$$
for all $n \geq N$, i.e.,
$$\sup_{i \in [n]}|\langle P^{*}\omega, P^{*}\eta \rangle_{\varepsilon_{n},n}(p_{i}) -\langle \iota^{*}\omega, \iota^{*}\eta \rangle_{g}(p_{i}) |
\xrightarrow{n \to \infty} 0 \text{ a.s.}$$
and so the convergence is uniform.




\end{proof}

\subsubsection{Global \( k \)FC Consistency}
\label{subsubsec:global-kfc-consistency}

For global considerations, we will assume $M$ is compact. This will enable uniform bounds on pointwise estimates that greatly simplify the upgrading of local $k$FC convergence to global $k$FC convergence. Fixing notation as in Section~\ref{subsubsec:discrete-global-inner-product}, let $\langle \! \langle \_, \_ \rangle \! \rangle_{\varepsilon, n}$ denote the global inner product with respect to the density-corrected measure, $$\langle \! \langle P^{*} \omega , P^{*}\eta \rangle \! \rangle  _{\varepsilon, n }=\frac{1}{n}\sum_{i=1}^{n}\big(\langle P^{*}\omega, P^{*} \eta \rangle(p_{i}) \big)  \, \frac{1}{q(p_{i})},$$
and let $\langle \! \langle \_,  \_ \rangle \! \rangle_{\varepsilon, n, q}$ denote the global inner product with respect to the uniform (non-density corrected) measure: $$\langle \! \langle P^{*}\omega, P^{*}\eta \rangle \! \rangle_{\varepsilon, n,q} = \frac{1}{n} \sum_{i=1}^{n} \langle P^{*}\omega, P^{*}\eta \rangle  (p_{i}).$$
Similarly, $\langle \! \langle \_,\_ \rangle \! \rangle_{g}$ denotes the continuous inner product on $M$ with respect to the standard volume form $dV$ and $\langle \! \langle \_, \_ \rangle \! \rangle_{g, q}$ the inner product with respect to a weighting $q\, dV$ on \( M \).


\begin{corollary}[Global Continuum Limit]

\label{cor:global-continuum-limit}
Assume \( M \) is \emph{compact}. Upon taking $\varepsilon_{n}:= n^{-\vartheta}$ for $0<\vartheta< \frac{2}{d+8}$, one has $$\lim_{n \to \infty} \langle \! \langle P^{*}\omega, P^{*} \eta \rangle \! \rangle_{\varepsilon_{n}, n} = \langle \! \langle \iota^{*} \omega, \iota^{*} \eta \rangle \! \rangle_{g} $$
and $$\lim_{n \to \infty} \langle \! \langle P^{*}\omega, P^{*}\eta \rangle  \! \rangle _{\varepsilon_{n}, n, q}=\langle \! \langle \iota^{*}\omega, \iota^{*}\eta \rangle \! \rangle_{g, q} $$
    almost surely.
\end{corollary}

\begin{proof}

    First consider that for $(\mu_{n})_{n=1}^{\infty}$ any sequence of measures on $[n]$ with $\sup_{n} \mu_{n}(P)<\infty$ and $\mu_{n} \to \mu$ for $\mu$ some measure on $M$, we may write \begin{align}
& \left | \langle \!  \langle P^{*}\omega, P^{*}\eta \rangle \! \rangle_{\varepsilon_{n}, n}   - \langle \! \langle \iota^{*}\omega, \iota^{*}\eta \rangle \! \rangle_{g}  \right|  = \left  | \sum_{i=1}^{n} \langle P^{*} \omega, P^{*}\eta \rangle(p_{i})\mu_{n}(p_{i}) - \int _{p \in M} F(p) \, d\mu       \right|, \, \, F(p):=\langle \iota^{*} \omega, \iota^{*} \eta \rangle_{g}(p), \\
& \leq \underbrace{ \sum_{i=1}^{n} |  \langle P^{*}\omega, P^{*}\eta \rangle_{\varepsilon_{n}, n} (p_{i} ) -   F(p_{i})  | \mu_{n}(p_{i})  }_{ \mathrm{(1)} }+ \underbrace{ | \sum_{i=1}^{n} F(p_{i}) \mu_{n}(p_{i} )  - \int _{p \in M} F \,d \mu  | }_{ \mathrm{(2)} }
\end{align}
where $(\mathrm{1}) \xrightarrow{n \to \infty} 0$ a.s. by Corollary~\ref{cor:local-continuum-limit}\eqref{item:local-continuum-limit-item-2} and $\sup_{n} \mu_{n}(P)<\infty$. It remains to show $\mathrm{(2)} \to 0$ almost surely; this follows from the law of large numbers and casework on $\mu_{n}$. 

\textit{$\mu_{n}=\operatorname{uniform}$}. Compactness of $M$ ensures $F$ is bounded, whence $\sum_{i=1}^{n} F(p_{i}) \mu_{n}(p_{i})=\frac{1}{n}\sum_{i=1}^{n} F(p_{i}) \to \int_{p \in M}F(p)q(p)\, dV(p)$ a.s. and the result follows.

\textit{$\mu_{n}=\operatorname{density-corrected}.$} Because $M$ is compact and $q>0$ is continuous, $\frac{1}{q}$ is bounded, hence $\mu_{n}=\frac{1}{n} \sum_{i=1}^{n} \frac{\delta}{q}$ satisfies $\sup_{n}\mu_{n}(P)<\infty$. Now $$\frac{1}{n} \sum_{i=1}^{n} \frac{F(p_{i})}{q(p_{i})} \to \int_{p \in M} \frac{F(p)}{q(p)}q(p)\, dV(p)=\int_{p \in M} F \, dV \text{ a.s.}$$
and the result follows.

\end{proof}

\begin{remark}
While the discussion and proof above hold assuming knowledge of the true \( q \) for clarity, the argument holds with apparent modifications upon replacing with an empirical measure \( \hat{q} _{n}\) provided \( \| \hat{q} _{n}- q\|_{L^\infty(M)}=O(\varepsilon_{n}) \).
\end{remark}

\subsubsection{Elaboration on Proposition~\ref{prop:BH-cor1}}
\label{subsubsec:the-bh-good-event}
\textit{Justification of Remark~\ref{rmk:whp-explanation}.} In Remark~\ref{rmk:whp-explanation}, we claimed without justification that conclusion of Proposition~\ref{prop:BH-cor1} holds on an event $E_{\varepsilon, n, i , f}=A_{\varepsilon, n} \cap B_{\varepsilon, n, i ,f}$ satisfying
\begin{equation}
\mathbb{P}(E_{\varepsilon, n, i, f}^{c}) \leq C_A n  \, e^{- c_{A}n \varepsilon^{4+ d/ 2}}+C_{B}e^{-c_{B,i, f}n \varepsilon^{3+d/2}}
\end{equation}
{for some constants $C_{A},C_{B},c_{A},c_{B,i, f}>0$ corresponding to the latter two terms of \( \delta_{\varepsilon, n, \alpha, \beta, f} \).}

Indeed, mining Appendix B of~\cite{Berry2016}, one finds the following.

\textbf{Setup.} Let \( S_{\varepsilon, \alpha, \beta} \) be the Markov operator corresponding to \( \hat{K} \), that is, \begin{align}
(S_{\varepsilon, \alpha, \beta} f)(p_{i})=\sum_{j=1}^{n} f(p_{j}) \hat{K}_{\varepsilon, \alpha, \beta}(p_{i}, \{ p_{j} \})&=\frac{\sum_{j=1}^n \overbrace{f(p_j)\,K_{\varepsilon,\alpha,\beta}(p_i,\{p_j\})}^{=: a_i^{-1} F_i(p_j)}}{\sum_{j=1}^n \underbrace{K_{\varepsilon,\alpha,\beta}(p_i,\{p_j\})}_{=: a_i ^{-1}G_i(p_j)}},  \, a_{i}= n \varepsilon ^{d/2} /  q_{\varepsilon, \beta}(p_{i})^\alpha.
\end{align}

Fix notation  $\overline{J}_{i} = \overline{J}(i)= \frac{1}{n-1} \sum_{j \neq i}J_{i}(j)$ for \( J_{i}\) a given function of \( j \in [n] \) . Note\footnote{We write '$\approx$' because notation $\overline{J}$ is strictly speaking leave-one-out average rather than the true averages used to define $S_{\varepsilon, \alpha, \beta}$. } 
$$\frac{\overline{F}_{i}}{\overline{G}_{i}}\approx S_{\varepsilon, \alpha, \beta }f (p_{i})$$
and so $\check{L}f(p_{i})=\frac{S_{\varepsilon, \alpha, \beta}f(p_{i})-f(p_{i})}{\varepsilon} \approx  \frac{\frac{\overline{F}_{i}}{\overline{G}_{i}}- f(p_{i})}{\varepsilon}$, whence

$$L_{\varepsilon, \alpha, \beta}f(p_{i})=\frac{1}{ m \rho(p_{i})^{2}} \check{L} f(p_{i}) \approx \frac{1}{m \rho(p_{i})^{2}\varepsilon}\left(  \frac{\overline{F}_{i}}{\overline{G}_{i}}- f(p_{i})  \right).$$
Replacing the empirical averages with population expectations, we make the definition
$$\mathcal{L}_{\varepsilon, \alpha, \beta }f(p_{i}):=\frac{1}{m \rho(p_{i})^{2}} \check{\mathcal{L}}f(p_{i}), \text{ where }\check{\mathcal{L}}f(p_{i}):=\frac{\frac{\mathbb{E}[F_{i}]}{\mathbb{E}[G_{i}]}- f(p_{i})}{\varepsilon}.$$ Also define $\varepsilon^{d/2}H_{j}(x_{\ell}):=\overbrace{  \frac{1}{\rho(p_{j})^{d}} k_{\varepsilon, \beta}(p_{j}, p_{\ell}) }^{ \ell \text{th}  \text{ term  of }q_{\varepsilon, \beta}(p_{j}) }$, so that $\overline{H}_{j}=\frac{1}{n-1} \varepsilon^{-d/2}(q_{\varepsilon, \beta}(p_{j})- H_{j}(p_{j}) \varepsilon^{d/2}) \approx \frac{\varepsilon^{-d/2}}{n}q_{\varepsilon, \beta}(p_{j})$. 

{\textbf{The event $A_{\varepsilon,n}$. (\cite{Berry2016}, Appendix B.1)}}  In Appendix B.1 of \cite{Berry2016} it is shown that the second term $\frac{q(p_{i})^{(1- d \beta)/2}}{\sqrt{ n }\varepsilon^{2+ d/ 4}}$ of $\delta_{\varepsilon, n, \alpha, \beta, f}(p)$ arises from the requirement that $\overline{H}_{j}$ matches the population expectation $\mathbb{E}[H_{j}]$ up to order-$\varepsilon^{2}$.\footnote{Explicitly, Appendix B.1 of~\cite{Berry2016} manufactures the estimate \( \frac{q(p_{i})^{1/2} \rho(p_{i})^{-d/2}}{\sqrt{n} \varepsilon ^{2+d/4}} \), into which we have substituted our assumed $\rho=q^{\beta}+O(\varepsilon)$.} For a single point $p_{j}$, the event $A_{\varepsilon, n , j}:=\{ \overline{H}_{j}= \mathbb{E}[H_{j}]+O(\varepsilon^{2}) \}$ has failure estimated by the pointwise Chernoff bound (\cite{Berry2016}, B.5)
\begin{equation}
\mathbb{P}( |\overline{H}_{j}- \mathbb{E}[H_{j}]|> a_{1} )\leq  2 e ^{\frac{- {a}_{1}^{2}(n-1)\varepsilon^{d/2}}{4 \hat{m}_{0} } \frac{\rho(p_{j})^{d}}{q(p_{j})} }, \hat{m}_{0}=\int_{\mathbb{R}^{d}}h(\|z\|^{2})^{2} \, dz, a_{1}=\hat{a}_{1}\varepsilon^{2}
\end{equation}
Writing $\zeta:=\zeta(\rho):=\inf_{p \in M} \frac{\rho(p)^{d}}{q(p)} \in [0, \infty]$, the uniform event $A_{\varepsilon, n}:= \bigcap_{j=1}^{n}A_{\varepsilon, n ,j}$ has failure estimated by the union bound
\begin{equation}
\mathbb{P}(\underbrace{ \text{For some }j \in [n], |\overline{H}_{j}- \mathbb{E}[H_{j}]|>\hat{a}_{1} \varepsilon^{2} }_{ A_{\varepsilon, n}^{c} }) \leq  2n\, e ^{\frac{- {a}_{1}^{2}(n-1)\varepsilon^{d/2}}{4 \hat{m}_{0} } \zeta }
\label{eqn:label}
\end{equation}
Note that substituting $\rho=q^{\beta}+O(\varepsilon)$ gives $\frac{\rho(p_{})^{d}}{q(p_{})}=q(p_{})^{d \beta -1}+O(\varepsilon)$; since $\beta<0$ and $\|q\|_{\infty}<\infty$ one has $\zeta=\zeta(\rho=q^{\beta}+O(\varepsilon)) \geq \|q\|_{\infty}^{d \beta - 1} >0$. Substituting $\rho=q^{\beta}+O(\varepsilon)$, $a_{1}=\hat{a}_{1}\varepsilon^{2}$ into \eqref{eqn:label} attains
\begin{equation}
\mathbb{P} (A_{\varepsilon, n}^{c}) \leq 2n \, e^{- c_{A}n \varepsilon^{4+ d/ 2}}
\end{equation}
for a constant $c_{A}=c_{A}(\hat{a}_{1}, \hat{m}_{0},  \|q\|_{\infty})>0$.


{\textbf{The event $B_{\varepsilon,n ,i ,f}$. (\cite{Berry2016} Appendix B.2)}}
Assume $A_{\varepsilon, n}$ holds. In Appendx B.2 of \cite{Berry2016} it is shown that, on $A_{\varepsilon,n}$, the third term $\frac{\|\nabla f(p_{i})\| q(p_{i})^{-c_{2}}}{\sqrt{ n } \varepsilon^{1/2+d/4}}$ in $\delta_{\varepsilon, n, \alpha, \beta, f}$ arises from the requirement that $L_{\varepsilon,\alpha, \beta }f(p_{i})$ matches $\mathcal{L}_{\varepsilon,\alpha, \beta}f(p_{i})$ to order-$\varepsilon$\footnote{Explicitly, Appendix B.2 of \cite{Berry2016} manufactures the estimate $\frac{\|\nabla f(p_{i})\|\, q(p_{i})^{-(1/2 - 2\alpha + 2d\alpha)}\, \rho(p_{i})^{-(d/2+1)}}{\sqrt{n}\, \varepsilon^{1/2 + d/4}}$, into which we have substituted our assumed $\rho = q^{\beta} + O(\varepsilon)$, collapsing the combined $q$-exponent to $-c_{2} = -(1/2 - 2\alpha + 2d\alpha + d\beta/2 + \beta)$.}, i.e., that the event
\begin{equation}
B_{\varepsilon, n ,i, f}:=\left\{ \overbrace{ \frac{1}{\varepsilon m \rho(p_{i})^{2}} \left| \frac{\overline{F}_{i}}{\overline{G}_{i}}- \frac{\mathbb{E}[F_{i}]}{\mathbb{E}[G_{i}]} \right| }^{= |L_{\varepsilon, \alpha , \beta}f(p_{i})-\mathcal{L}_{\varepsilon, \alpha, \beta}f(p_{i})| } \leq  \hat{a}_{2}\varepsilon \right\}
\end{equation}
holds. Appendix B.2 of \cite{Berry2016} proceeds to justify that
\begin{equation}
\mathbb{P}\left\{ \overbrace{ \frac{1}{\varepsilon m \rho(p_{i})^{2}} \left| \frac{\overline{F}_{i}}{\overline{G}_{i}}- \frac{\mathbb{E}[F_{i}]}{\mathbb{E}[G_{i}]} \right| }^{= |L_{\varepsilon, \alpha , \beta}f(p_{i})-\mathcal{L}_{\varepsilon, \alpha, \beta}f(p_{i})| }  >  \hat{a}_{2}\varepsilon \right\}=\mathbb{P} (\overline{Y}_{j}> a(n-1) \mathbb{E}[G_{i}]^{2} \varepsilon m \rho^{2})
\end{equation}
for any $0<a< 1$, where $Y_{j}:=\mathbb{E}[G_{i}] F_{i}(p_{j})- \mathbb{E}[F_{i}]G_{i}(p_{j})+ a \varepsilon m \rho(p_{i})^{2} \mathbb{E}[G_{i}](\mathbb{E}[G_{i}]-G_{i}(p_{j}))$as defined in~\cite{Berry2016}. Using the variance estimate (B.10) in~\cite{Berry2016} and subsequent Chernoff bound with $a=\hat{a}_{2} \varepsilon$, one has
\begin{equation}
\mathbb{P}\big((B_{\varepsilon,n,i,f}^{+})^{c}\big) \leq 2 e^{-\hat{a}_{2}^{2}(n-1)c \rho(p_{i})^{d+2}q(p_{i})^{1+4 \alpha(d-1) } \|\nabla f(p_{i})\|^{-2}\varepsilon^{3+d/2}}
\end{equation}
where $B_{\varepsilon, n ,i, f}^{+}$ is the event $\left\{ \frac{1}{\varepsilon m \rho(p_{i})^{2}} ^{}\left(  \frac{\overline{F}}{\overline{G}}- \frac{\mathbb{E}[F]}{\mathbb{E}[G]}  \right) \leq   \hat{a}_{2} \varepsilon  \right\}$ and $c:=m^{2}m_{0}^{2+6 \alpha} / 8 \hat{m}_{2}$ with $\hat{m}_{2}:=\int_{\mathbb{R}^{d}}z_{1}^{2}h(\|z\|^{2})^{2}\, dz$. Substituting our $\rho=q^{\beta}+O(\varepsilon)$, using $2c_{2}=1+4\alpha(d-1)+\beta(d+2)$, this becomes
\begin{equation}
\mathbb{P}\big((B_{\varepsilon,n,i,f}^{+})^{c}\big) \leq 2 e^{ -\hat{a}_{2}^{2} (n-1)c q(p_{i})^{2c_{2}} \|\nabla f (p_{i})\|^{-2} \varepsilon^{3+d/2}  }.
\end{equation}
Applying the same estimate with $f$ replaced by $-f$ (whose gradient has the same magnitude) and union bounding gives
\begin{equation}
\mathbb{P}\big( B_{\varepsilon, n, i, f}^{c} \big) \leq 4 e^{ -\hat{a}_{2}^{2} (n-1)c q(p_{i})^{2c_{2}} \|\nabla f (p_{i})\|^{-2} \varepsilon^{3+d/2}  }=4 e^{-c_{B,i, f}n \varepsilon^{3+d/2}}
\end{equation}
for a constant $c_{B,i,f}=c_{B,i,f}(\hat{a}_{2}, c_{2}, c, \|\nabla f(p_{i})\|, q(p_{i}))>0$.

Defining $E_{\varepsilon,n, i, f}:=A_{\varepsilon, n} \cap B_{\varepsilon, n, i, f}$ and union bounding the complement now gives the claimed failure rates.



\section{Implementation and Experimental Details}
\label{appendix:C}
\label{app:implementation-details}
\providecommand{\asinh}{\operatorname{asinh}}

This appendix records the implementation choices and response-task construction details needed to reproduce the experiments in Section~\ref{sec:experiments}. The compiled paper uses this appendix as the canonical implementation appendix: it summarizes the point-form layer, the variant sweeps, the PDO response data pipeline, and the storage audit for the \(G_P^{(k)}\) tensors.

\subsection{Point-form layer}

For each cloud \(P=\{p_i\}_{i=1}^{m}\subset\mathbb{R}^{D}\), the implementation first builds the discrete extrinsic Gram field \(G_P^{(k)}\) described in Section~\ref{Appendix:subsec:discrete-kfcs}. A neural coefficient map
\[
F_{\theta}:\mathbb{R}^{D}\to \mathbb{R}^{\ell\times {D \choose k}}
\]
then produces \(\ell\) learned ambient \(k\)-forms at every point. The layer returns the comparison matrix
\[
C_P(\theta) = \sum_{p\in P}\mu_P(p)\, F_{\theta}(p)G_P^{(k)}(p)F_{\theta}(p)^{\top} \in\mathbb{R}^{\ell\times \ell}.
\]
Downstream classifiers consume this symmetric matrix, either after flattening the upper triangle or after applying a small matrix readout.

\begin{table}[H]
\caption{Implemented variant factors. Each row is a modeling factor varied in the local sweeps; The purpose of the factors is to separate the geometric degree \(k\), the estimator policy, and the final matrix readout.}
\label{tab:implementation-variant-factors}
\centering
\small
\setlength{\tabcolsep}{4pt}
\renewcommand{\arraystretch}{1.12}
\begin{tabular}{p{0.21\linewidth}p{0.29\linewidth}p{0.40\linewidth}}
\toprule
Factor & Values used & Role \\
\midrule
Form degree & \(k=1\), \(k=2\) & Tests vector-like and area-like cotangent comparisons. \\
Readout layout & diag, flat, pool, tri & Controls whether the classifier sees diagonal energies, flattened comparison entries, pooled summaries, or upper-triangular entries. \\
Bandwidth policy & fixed, variable & Switches between a global graph scale and a local variable-bandwidth diffusion scale. \\
Density policy & fixed, variable & Switches between uncorrected and density-adapted weighting in the point-cloud estimator. \\
Graph construction & \(k\)-NN truncation & Used for computational efficiency, while Appendix~\ref{Appendix:Subsec:theory} states the corresponding continuum assumptions. \\
\bottomrule
\end{tabular}
\end{table}

\begin{table}[H]
\caption{Implementation-facing role of the reported task families. Synthetic tasks expose known geometric or dynamical relations; PDO tasks test whether the same comparison matrix is useful when the governing biological relation is unknown and labels are deterministic silver labels.}
\label{tab:implementation-task-roles}
\centering
\small
\setlength{\tabcolsep}{4pt}
\renewcommand{\arraystretch}{1.12}
\begin{tabular}{p{0.22\linewidth}p{0.30\linewidth}p{0.38\linewidth}}
\toprule
Task family & Implementation stress & Reporting caveat \\
\midrule
Line--circle controls & \(1\)-form recovery of known tangent directions in \(\mathbb{R}^{2}\) & The label is the generator relation, not a hidden metadata coordinate. \\
Density-shift controls & Bandwidth and density policies when sampling density is a nuisance & Density correction should help here, but the same interpretation need not hold for biological clouds. \\
RNA kinetics & Variable-length trajectory clouds generated by perturbed dynamical systems & Synthetic analogue of RNA-velocity reasoning; not real cell-fate inference. \\
PDO tumor-selective response & Response-coordinate clouds with death-marker leakage controls\citep{kaufman2011leakage, kapoor2023leakage} & Matched tumor-versus-fibroblast contrast with death markers withheld from inputs. \\
PDO EFP & Response-coordinate clouds with target-module exclusions & High epithelial functional perturbation is a benchmark threshold, not a clinical cutoff. \\
PDO CAF & Microenvironment-conditioned epithelial clouds & CAF-mediated reprogramming is a deterministic silver label derived from matched condition contrasts. \\
\bottomrule
\end{tabular}
\end{table}

\subsection{Baselines}

The reported comparisons use three baseline families. Classical ML baselines operate on pooled cloud summaries and include logistic regression, random forest, and RBF-kernel SVM \citep{cox1958regression,breiman2001random,cortes1995support}. Neural point-cloud and graph baselines include MLPs \citep{rumelhart1986learning}, PointNet-style models and PointNet++ \citep{qi2017pointnet,qi2017pointnet++}, PointTransformer \citep{zhao2021point}, GCN \citep{kipf2017semisupervised}, GIN \citep{xu2018how}, GraphSAGE \citep{hamilton2018inductiverepresentationlearninglarge}, GraphTransformer \citep{dwivedi2021generalizationtransformernetworksgraphs}, and TDL models when the corresponding input representation was available \citep{papamarkou2024position,ebli2020simplicial,bodnar2021weisfeiler,bodnar2021cellular,hajij2022topological}. Hodge baselines use direct or Hodge-style \(k\)-order features without the learned NPF comparison layer \citep{hirani2003discrete,desbrun2003discrete,arnold2018finite,maggs2023simplicial}.

\subsection{Hyperparameter policy}

All model families are evaluated on the same task folds and seeds within a given aggregate table. Classical baselines are reported without a neural epoch count. Neural and \kfc{} rows use the completed sweep values stored in the normalized result tables; the main PDO response rows in Table~\ref{tab:pdo-main-response} correspond to the best completed \kfc{} hyperparameter row for each task. Density-shift and RNA-kinetic tasks are interpreted as synthetic geometry checks, and the three PDO response rows are the biological response tasks. We do not use the main text to report every attempted \kfc{} cell; instead, Table~\ref{tab:experiment-accounting} records the aggregate sweep sizes used to audit the reported rows.

\begin{table}[H]
\caption{Experiment accounting for the paper-facing sweeps. A manifest row denotes one model/hyperparameter configuration before expansion over the reported folds and seeds; the main sweeps use five folds and five seeds.}
\label{tab:experiment-accounting}
\centering
\small
\setlength{\tabcolsep}{4pt}
\renewcommand{\arraystretch}{1.12}
\begin{tabular}{p{0.30\linewidth}rrp{0.36\linewidth}}
\toprule
Sweep or task & Manifest rows & Fold--seed fits & Scope \\
\midrule
Synthetic primary sweep & 336 & 8{,}400 & Eight synthetic/control tasks with 42 model rows per task. \\
Synthetic NPF hyperparameter screen & 2{,}560 & 64{,}000 & Five controlled synthetic tasks with 512 NPF rows per task. \\
PDO tumor-selective primary rows & 42 & 1{,}050 & Baselines plus selected NPF rows for the response table. \\
PDO tumor-selective NPF screen & 512 & 12{,}800 & Completed NPF variants for response-task selection. \\
PDO EFP primary rows & 42 & 1{,}050 & Baselines plus selected NPF rows for the response table. \\
PDO EFP NPF screen & 512 & 12{,}800 & Completed NPF variants for response-task selection. \\
PDO CAF primary rows & 42 & 1{,}050 & Baselines plus selected NPF rows for the response table. \\
PDO CAF NPF screen & 512 & 12{,}800 & Completed NPF variants for response-task selection. \\
\midrule
\textbf{Total} & \textbf{4{,}558} & \textbf{113{,}950} & \\
\bottomrule
\end{tabular}
\end{table}

\subsection{PDO response dataset, labels, and splits}
\label{subsec:pdo-response-data-construction}

The PDO response tasks are built from the public patient-derived organoid and cancer-associated fibroblast single-cell mass-cytometry\citep{bendall2011masscytometry, nowicka2017cytof} release associated with the Trellis PDO study \citep{RamosZapatero2023,mendeley2022pdo}. The upstream record is available through Mendeley Data, DOI \url{https://doi.org/10.17632/hc8gxwks3p.1}. The working source in our repository is the curated PDO/CAF single-cell table under \texttt{data/phi\_pdo/PDO\_CAF\_cleaned.parquet}. 

Each row of the source table is a measured cell with marker intensities and experimental design fields. The preparation step first retains the required design metadata -- patient, treatment, culture context, cell compartment, and concentration or plate information when available -- and treats these fields as metadata rather than marker coordinates. Measured marker columns are coerced to numeric values, rows with non-finite retained markers are removed, and marker intensities are transformed with the fixed cytometry scale \(x\mapsto \asinh(x/5)\). No full-dataset fitted normalizer is applied before splitting. The supervised unit is then a condition-level empirical measure, not an individual cell: a point is one transformed cell profile, and a sample is a cloud of cells from one eligible treatment condition.

For the reported response tasks, labels are computed from matched treatment-control contrasts\citep{vandewetering2015organoidbiobank,
    vlachogiannis2018organoids, RamosZapatero2023}. Matching uses available design fields such as patient, plate, culture context, cell compartment, treatment, concentration, and the appropriate vehicle/control condition. The matched control is used only to compute the response score (it is not supplied to the model as a paired input). After excluding task-defining marker modules, the remaining permitted markers are projected to twelve response coordinates, so each reported response cache contains \(240\) clouds with \(256\) sampled cells per cloud and ambient dimension \(D=12\).


The labels are deterministic silver labels computed from measured markers and metadata, not clinical-outcome annotations. Epithelial functional perturbation (EFP) uses matched treated-minus-control changes in apoptosis, DNA-damage, and cell-cycle modules\citep{nicholson1995apopain, tewari1995yama,
    jackson2009dnadamage, lord2012dnatherapy, malumbres2009cellcycle, hanahan2022hallmarks}. After robust centering and scaling of eligible module deltas, apoptosis and DNA-damage changes contribute positively to the EFP score, while cell-cycle change contributes negatively. The binary EFP label is \(1\) for the top \(30\%\) of EFP scores\footnote{The top 30\% thresholds and the stric auxiliary cutoffs 0.05 and 0.02 are not literature-derived biological thresholds} and \(0\) otherwise; apoptosis, PARP/caspase, DNA-damage, and cell-cycle target modules are withheld from the input coordinates. 
Tumor-selective response uses the therapeutic index
\[
\mathrm{TI}=\Delta_{\mathrm{PDO,death}}-\max(\Delta_{\mathrm{Fib,death}},0),
\qquad
y_{\mathrm{response}}=\mathbf{1}\{\mathrm{TI}>0\},
\]
where death is summarized from transformed \(\mathrm{cCaspase\_3}\)\citep{nicholson1995apopain} and \(\mathrm{cPARP}\)\citep{tewari1995yama}. A positive PDO death delta is rewarded, whereas positive fibroblast death is subtracted as a toxicity penalty\citep{ RamosZapatero2023, kalluri2016fibroblasts, sahai2020caf}. The direct death markers are excluded from the model inputs. The prepared cache also stores a stricter auxiliary label requiring \(\Delta_{\mathrm{PDO,death}}\geq0.05\) and \(\mathrm{TI}\geq0.02\), but the main benchmark target is \(y_{\mathrm{response}}\).

CAF-mediated epithelial reprogramming compares matched PDOF co-culture and PDO monoculture effects. The CAFRep score \citep{thiery2009emt, manning2007akt,
    malumbres2009cellcycle} rewards loss of epithelial state, gain of mesenchymal state, gain of survival signaling, and reduced cycling relative to the matched monoculture contrast. The binary CAF label is \(1\) for the top \(30\%\) of CAFRep scores and \(0\) otherwise. As with EFP, the top-\(30\%\) threshold is a benchmark-construction rule for a high-response class, not a clinical cutoff, and label-module markers are withheld when available under the strict input policy \cite{RamosZapatero2023, kalluri2016fibroblasts, sahai2020caf}.

Splits are defined at the cloud level. The response loaders attach a deterministic train/validation/test assignment with seed \(0\), using a \(60/20/20\) target split, and also attach five cross-validation fold IDs. Both the holdout split and the CV folds are grouped by patient, so a patient cannot contribute response clouds to both sides of a split or validation fold. All model families consume the same stored fold IDs in the shared evaluation workflow; the main tables therefore compare models on identical patient-grouped folds rather than on model-specific random splits.

\subsection{Toy model: RNA kinetics}

\paragraph{Setup} We simulated synthetic RNA-velocity trajectories for $n$ genes using a transcription-on kinetic model in which, for each gene $i\in\{1,\dots,n\}$, the unspliced and spliced abundances satisfy
$$\frac{du_i}{dt}=\alpha_i-\beta_i u_i,\qquad
\frac{ds_i}{dt}=\beta_i u_i-\gamma_i s_i.$$
The full state is $$x(t)=(u_1(t),\dots,u_n(t),s_1(t),\dots,s_n(t))\in\mathbb{R}^{2n}$$ and each class $c\in\{0,1\}$ is defined by a vector field $F_c:\mathbb{R}^{2n}\to\mathbb{R}^{2n}$ determined by class-specific parameter vectors $\alpha^{(c)},\beta^{(c)},\gamma^{(c)}$. To make the classification problem nontrivial, the two classes differed only on a small subset of genes; for those genes, the class-1 parameters were generated by additive perturbations of the class-0 parameters, while the remaining genes were unchanged.

\paragraph{Trajectories} Initial conditions $x_0^{(j)}$ were sampled near the steady state of the base field and the same set of initial conditions was reused for both classes, ensuring that labels depended on the class-specific dynamics rather than on different seed distributions. For each trajectory, within-class heterogeneity was introduced by multiplicative perturbations of the kinetic parameters,

$$\tilde\alpha_i=\alpha_i e^{\varepsilon_i^\alpha},\qquad
\tilde\beta_i=\beta_i e^{\varepsilon_i^\beta},\qquad
\tilde\gamma_i=\gamma_i e^{\varepsilon_i^\gamma},$$
where the perturbations were independent zero-mean Gaussian random variables. For each initial condition and class, the ODE $\dot x=\tilde F_c(x)$ was numerically integrated over a short interval $t\in[t_0,t_1]$, yielding sampled states $x(t_0),\dots,x(t_{m-1})$. Observation noise was then added according to
$$y_k=x(t_k)+\eta_k,\qquad \eta_k\sim\mathcal N(0,\sigma_{\mathrm{obs}}^2 I),$$
followed by temporal subsampling and truncation to a $k_i\sim\mathrm{Uniform}(k_{\min},k_{\max})$, so that each sample consisted of a short early-time trajectory fragment of variable length.

\subsection{Computational costs and resource accounting}
\label{subsec:computational-costs}

 The model runs leverage CUDA/NCCL. Runs were scheduled on one node with eight H100-class GPUs available concurrently. The scheduler used GPU tokens \(0,\ldots,7\) and at most eight active jobs. The table reports configuration and fold--seed counts rather than a cloud-dollar cost; no provider billing rate was attached to the experiment manifests.

\begin{tcolorbox}[
enhanced, breakable, colback=gray!4, colframe=black!45, boxrule=0.4pt, arc=2pt, left=5pt, right=5pt, top=4pt, bottom=4pt, title=Computational remark ] The largest cost is forming and storing \(G_P^{(k)}\). For a cloud with \(m\) points in ambient dimension \(D\), the dense tensor has shape \(m\times {D\choose k}\times {D\choose k}\), so an fp32 materialization requires \(4m{D\choose k}^{2}\) bytes. Consequently \(k=2\) is already substantially more expensive than \(k=1\), and \(k=3\) becomes a diagnostic stress case rather than a routine setting for high-dimensional PDO or single-cell clouds. This is why the reported experiments emphasize \(k=1\), targeted \(k=2\) variants, and reduced response-coordinate bundles for the PDO response tasks.
\end{tcolorbox}

\paragraph{Accounting convention.} We distinguish allocated GPU-hours from row-process GPU-hours. Allocated GPU-hours are the reproducibility and resource-planning quantity: wall time multiplied by the physical GPUs allocated to the run. Row-process GPU-hours sum per-row runtimes over concurrently dispatched workers. They are useful for measuring workload breadth, but they are not a billing quantity when the dispatcher oversubscribes workers relative to physical GPUs.

\paragraph{Complexity Guide.}
Let $N$ be the number of clouds, $m$ the number of points per cloud, $n$ the ambient dimension, $k$ the number of graph neighbors, $L$ the number of learned forms, $q$ the exterior form degree, $E$ the number of epochs, $F$ the number of folds, and $S$ the number of seeds. Graph construction uses dense pairwise distances via \texttt{torch.cdist}, followed by \texttt{topk}; thus, per cloud, it costs roughly $O(m^2 n)$ time and $O(m^2)$ transient memory, while the retained sparse graph costs $O(mk)$ memory. See \texttt{src/volrep\_forms/base/maths/geometry.py:510} and \texttt{src/volrep\_forms/base/maths/geometry.py:541}. The geometric precompute builds the ambient CdC metric $G_1$ once per cloud using about $n(n+1)/2$ sparse matrix multiplications, giving roughly $O(n^2 m k)$ time and $O(m n^2)$ memory. See \texttt{src/volrep\_forms/base/data/pointcloud.py:218}. The standard $k=1$ VolRep forward path is
\[
\text{points} \to \text{learned 1-forms} \to \text{pointwise Gram} \to \text{pooled/global Gram} \to \text{MLP head}.
\]
The Gram contraction is the main model-specific cost, about $O(m L^2 n^2)$ naively, plus the form-MLP cost. See \texttt{src/volrep\_forms/training/trainer.py:261} and \texttt{src/volrep\_forms/base/data/pointcloud.py:309}. For $q \geq 2$, the compound $k$-form path grows combinatorially with $r = \binom{n}{q}$. The compound metric costs $O(m r^2)$ storage, and the contraction costs roughly $O(m L^2 r^2)$, which is why high-degree and high-ambient variants are the dangerous ones. See \texttt{src/volrep\_forms/training/trainer.py:267} and \texttt{src/volrep\_forms/base/maths/geometry.py:406}. Overall training multiplies the per-cloud cost by $E \times F \times S \times N \times \texttt{number\_of\_manifest\_rows}$. The GPU-hour cost is therefore dominated by manifest breadth. 

\paragraph{Compute Budget and Hardware.}
For the current final-paper run, we report that at least $1{,}195$ allocated GPU-hours have already been consumed, but we do not claim a final total since much of this has been under experimentation with higher \(k\). A clean full reproduction on comparable hardware should reserve approximately $2{,}000$--$3{,}500$ allocated GPU-hours to cover all baselines and models for full reproducibility. For hardware, the PDO-only setting is feasible on a single $8{\times}$A100, H100, or H200 node with at least 80GB GPU memory per device. The full final run should use at least the current layout, namely $2{\times}$ $8$-GPU A100-80GB nodes plus $1{\times}$H200. For fewer out-of-memory retries, we recommend $3{\times}$ $8$-GPU H100/H200 or A100-80GB nodes.

\paragraph{Higher $k$.} As stated before, as $k$ increases, so does the computational cost. The Table~\ref{tab:compound-gram-memory} outlines and approximation of the total required GPU memory to process our datasets with $k={1,2,3}$.
\paragraph{Compound Gram Memory Footprint.}
Table~\ref{tab:compound-gram-memory} reports the fp32 memory footprint of the Gram tensors used by the point-form models. For a point cloud with $m$ points in ambient dimension $n$, the $q$-form compound Gram tensor has shape
\[
[m,\binom{n}{q},\binom{n}{q}],
\]
so its fp32 memory cost is
\[
4m\binom{n}{q}^2 \text{ bytes}.
\]
Thus, while the $0$-form and $1$-form tensors remain small, the $2$-form and especially $3$-form compound tensors can become memory-dominant when $n$ is large. The table distinguishes the largest per-cloud tensor size from the hypothetical total memory required if all clouds in a task had their corresponding tensors materialized simultaneously.

\begin{table}[t]
\centering
\small
\setlength{\tabcolsep}{3pt}
\resizebox{\textwidth}{!}{%
\begin{tabular}{lrrrrrrrrrrr}
\toprule
Task & Ambient $n$ & Clouds & $m$ min/med/max & $C_2$ & $C_3$ &
$G_0$ total & $G_1$ total & $G_2$/cloud max & $G_2$ total &
$G_3$/cloud max & $G_3$ total \\
\midrule
ts2\_5\_concentric\_circles & 2 & 300 & 128 & 1 & 0 &
0.10 MB & 0.60 MB & 0.0005 MB & 0.10 MB & --- & --- \\
ts2\_5\_parallel\_lines & 2 & 300 & 128 & 1 & 0 &
0.10 MB & 0.60 MB & 0.0005 MB & 0.10 MB & --- & --- \\
ts2\_5\_colocated & 2 & 300 & 128 & 1 & 0 &
0.10 MB & 0.60 MB & 0.0005 MB & 0.10 MB & --- & --- \\
e10\_sanity & 2 & 300 & 128 & 1 & 0 &
0.10 MB & 0.60 MB & 0.0005 MB & 0.10 MB & --- & --- \\
ts25\_density\_shift & 2 & 296 & 256 & 1 & 0 &
0.30 MB & 1.20 MB & 0.001 MB & 0.30 MB & --- & --- \\
e10\_sanity & 2 & 300 & 128 & 1 & 0 &
0.10 MB & 0.60 MB & 0.0005 MB & 0.10 MB & --- & --- \\
\midrule
ts4\_singlecell & 50 & 320 & 128 & 1225 & 19600 &
0.20 MB & 0.38 GB & 0.72 GB & 228.98 GB & 183.18 GB & 58.62 TB \\
ts8\_rna\_velocity\_default & 50 & 1000 & 14754 & 1225 & 19600 &
0.10 MB & 0.21 GB & 0.22 GB & 125.93 GB & 57.24 GB & 32.24 TB \\
ts8\_rna\_velocity\_quick & 24 & 160 & 42104 & 276 & 2024 &
0.006 MB & 3.40 MB & 4.40 MB & 0.44 GB & 0.23 GB & 23.43 GB \\
pdo\_A\_celltype & 45 & 1079 & 17/128/128 & 990 & 14190 &
0.50 MB & 0.89 GB & 0.47 GB & 432.03 GB & 96.01 GB & 88.76 TB \\
pdo\_B\_culture & 45 & 1079 & 17/128/128 & 990 & 14190 &
0.50 MB & 0.89 GB & 0.47 GB & 432.03 GB & 96.01 GB & 88.76 TB \\
pdo\_C\_treatment & 45 & 760 & 21/128/128 & 990 & 14190 &
0.30 MB & 0.63 GB & 0.47 GB & 304.86 GB & 96.01 GB & 62.63 TB \\
pdo\_F\_patient & 45 & 1079 & 17/128/128 & 990 & 14190 &
0.50 MB & 0.89 GB & 0.47 GB & 432.03 GB & 96.01 GB & 88.76 TB \\
pdo\_A\_nocasp\_alias\_current & 45 & 1079 & 17/128/128 & 990 & 14190 &
0.50 MB & 0.89 GB & 0.47 GB & 432.03 GB & 96.01 GB & 88.76 TB \\
pdo\_A\_nocasp\_if\_44D & 44 & 1079 & 17/128/128 & 946 & 13244 &
0.50 MB & 0.85 GB & 0.43 GB & 394.48 GB & 83.64 GB & 77.32 TB \\
pdo\_response\_efp & 12 & 240 & 256 & 66 & 220 &
0.20 MB & 33.80 MB & 4.30 MB & 1.00 GB & 47.30 MB & 11.08 GB \\
pdo\_response\_tumor\_selective & 12 & 240 & 256 & 66 & 220 &
0.20 MB & 33.80 MB & 4.30 MB & 1.00 GB & 47.30 MB & 11.08 GB \\
pdo\_response\_caf\_reprogramming & 12 & 240 & 256 & 66 & 220 &
0.20 MB & 33.80 MB & 4.30 MB & 1.00 GB & 47.30 MB & 11.08 GB \\
\bottomrule
\end{tabular}%
}
\caption{Estimated fp32 memory footprint of Gram tensors across tasks. Here $C_2=\binom{n}{2}$ and $C_3=\binom{n}{3}$ denote the number of basis $2$-forms and $3$-forms in ambient dimension $n$. $G_0$ denotes the scalar $0$-form Gram tensor, whose memory scales as $m$ per cloud. $G_1$ denotes the $1$-form Gram tensor, with per-cloud shape $[m,n,n]$. $G_2$ and $G_3$ denote compound Gram tensors with per-cloud shapes $[m,\binom{n}{2},\binom{n}{2}]$ and $[m,\binom{n}{3},\binom{n}{3}]$, respectively. The ``cloud max'' columns report the memory of the largest single cloud, while the ``total'' columns report the memory if all clouds in the task had the corresponding tensors materialized at once.}
\label{tab:compound-gram-memory}
\end{table}

\end{document}